\documentclass[preprint,11pt]{elsarticle}

\usepackage{xcolor}

\usepackage{comment}

\usepackage{amsmath, amsxtra, amsfonts, amssymb, amstext,amsthm}
\usepackage{mathtools}
\usepackage{nicefrac}
\usepackage{xspace}
\usepackage[algo2e,ruled,vlined,linesnumbered]{algorithm2e}
\usepackage{algorithm}
\usepackage{algpseudocode}
\usepackage{multirow}
\usepackage{footnote}
\usepackage{hyperref}
\usepackage{thm-restate}

\usepackage{etoolbox} 

\usepackage{enumerate}
\usepackage{cases}
\usepackage{babel}
\usepackage{tikz}

\usepackage{graphicx}

\newcommand{\ooea}{$(1 + 1)$-EA\xspace}

\newcommand{\oclea}{$(1 , \lambda)$-EA\xspace}
\newcommand{\moea}{$(\mu + 1)$-EA\xspace}

\newcommand{\saolea}{\text{SA-}$(1, \lambda)$-EA\xspace}
\newcommand{\saoclea}{\text{SA-}$(1, \lambda)$-EA\xspace}
\newcommand{\saoplea}{\text{SA-}$(1+ \lambda)$-EA\xspace}
\newcommand{\saoleaparam}{$(1, \{\lambda/F,F^{1/s}\lambda\})$-EA\xspace}

\newcommand{\OneMax}{\textsc{OneMax}\xspace}
\newcommand{\onemax}{\textsc{OneMax}\xspace}

\newcommand{\zeromax}{\textsc{ZeroMax}\xspace}
\newcommand{\OM}{\textsc{Om}\xspace}
\newcommand{\ZM}{\textsc{Zm}\xspace}
\newcommand{\jump}{\textsc{Jump}\xspace}
\newcommand{\cliff}{\textsc{Cliff}\xspace}
\newcommand{\ridge}{\textsc{Ridge}\xspace}
\newcommand{\BinaryValue}{\textsc{BinaryValue}\xspace}

\newcommand{\dynBV}{\textsc{Dynamic BinVal}\xspace}

\newcommand{\leadingones}{\textsc{LeadingOnes}\xspace}

\newcommand{\hottopic}{\textsc{HotTopic}\xspace}
\newcommand{\binary}{\textsc{Binary}\xspace}

\newcommand{\round}[1]{\ensuremath{\left\lfloor #1 \right\rceil}}
\newcommand{\linit}{\ensuremath{\lambda^{\mathrm{init}}}}
\newcommand{\xinit}{\ensuremath{x^{\mathrm{init}}}}

\newcommand{\EE}{\ensuremath{\mathbf{E}}}
\def\erw{{\ensuremath{\mathbf{E}}}}
\newcommand{\indicator}[1]{\ensuremath{\mathbf{1}_{#1}}}
\newcommand{\cF}{\ensuremath{\mathcal{F}}}

\DeclareMathOperator*{\argmax}{arg\,max}
\newcommand{\supp}{\ensuremath{\mathrm{supp}}}
\newcommand{\RR}{\ensuremath{\mathbb{R}}}
\newcommand{\NN}{\ensuremath{\mathbb{N}}}

\newcommand\numberthis{\addtocounter{equation}{1}\tag{\theequation}}

\newcommand\eps{\varepsilon}

\newtheorem{theorem}{Theorem}
\newtheorem{definition}[theorem]{Definition}
\newtheorem{lemma}[theorem]{Lemma}

\newtheorem{sclaim}[theorem]{Claim}{\bfseries}{\itshape}
\newtheorem{corollary}[theorem]{Corollary}
\newtheorem{remark}[theorem]{Remark}

\newcommand{\case}[1]{\textbullet~#1:}

\journal{Theoretical Computer Science}

\begin{document}

\begin{frontmatter}



\title{Self-adjusting Population Sizes for the $(1, \lambda)$-EA on Monotone Functions}


\affiliation[eth]{organization={Department of Computer Science},
            addressline={ETH Z\"urich}, 
             city={Z\"urich},
             country={Switzerland}
            }
            

\author[eth]{Marc Kaufmann\corref{grantmarc}}

\ead{marc.kaufmann@inf.ethz.ch}

\cortext[grantmarc]{The author was supported by the Swiss National Science Foundation [grant number 192079].}

\author[eth]{Maxime Larcher}

\ead{larcherm@inf.ethz.ch}

\author[eth]{Johannes Lengler}

\ead{johannes.lengler@inf.ethz.ch}

\author[eth]{Xun Zou\corref{grantxun}}

\ead{xun.zou@inf.ethz.ch}

\cortext[grantxun]{The author was supported by the Swiss National Science Foundation [grant number CR-SII5\_173721].}

\begin{abstract}
We study the $(1,\lambda)$-EA with mutation rate $c/n$ for $c\le 1$, where the popula\-tion size is adaptively controlled with the $(1:s+1)$-success rule. Recently, Hevia Fajardo and Sudholt have shown that this setup with $c=1$ is efficient on \onemax for $s<1$, but inefficient if $s \ge 18$. Surprisingly, the hardest part is not close to the optimum, but rather at linear distance. We show that this behavior is not specific to \onemax. If $s$ is small, then the algorithm is efficient on all monotone functions, and if $s$ is large, then it needs super-polynomial time on all monotone functions. In the former case, for $c<1$ we show a $O(n)$ upper bound for the number of generations and $O(n\log n)$ for the number of function evaluations, and for $c=1$ we show $O(n\log n)$ generations and $O(n^2\log\log n)$ evaluations. We also show formally that optimization is always fast, regardless of $s$, if the algorithm starts in proximity of the optimum. All results also hold in a dynamic environment where the fitness function changes in each generation.

An extended abstract, containing only the results without proofs, has been published at the PPSN conference~\cite{kaufmann2022self}. 

\end{abstract}


\end{frontmatter}



\section{Introduction}
Randomized Optimization Heuristics (ROHs) like evolutionary algor\-ithms (EAs) are simple general-purpose optimizers. One of their strengths is that they can often be applied with little adaptation to the problem at hand. However, ROHs usually come with parameters, and their efficiency often depends on the parameter settings. Therefore, \emph{parameter control} is a classical topic in the design and analysis of ROHs~\cite{Eiben99parameter}. It aims at providing methods to automatically tune parameters over the course of optimization. The goal is not to remove parameters altogether; the parameter control mechanisms themselves introduce new meta-parameters. Nevertheless, there are two objectives that can sometimes be achieved with parameter control mechanisms.

Firstly, some ROHs are rather sensitive to small changes in the parame\-ters, and inadequate setting can slow down or even prevent success. Two examples that are relevant for this paper are the \oclea, which fails to optimize even the easy \OneMax benchmark if $\lambda$ is too small~\cite{jagerskupper2007plus, rowe2014choice, antipov2019efficiency}, and the \ooea, which fails on monotone functions if the mutation rate is too large~\cite{doerr2010optimizing, doerr2013mutation, lengler2018drift}. In both cases, changing the parameters just by a constant factor makes all the difference between finding the optimum in time $O(n\log n)$, and not even finding an $\eps$-approximation of the optimal solution in polynomial time. So these algorithms are extremely sensitive to small changes of parameters. In such cases, one hopes that performance is \emph{more robust} with respect to the meta-parameters, i.e., that the parameter control mechanism manages to find a decent parameter setting regardless of its meta-parameters.

Secondly, often there is no single parameter setting that is optimal throughout the course of optimization. Instead, different phases of the optimization process profit from different parameter settings, and the overall performance with \emph{dyna\-mically adapted} parameters is better than for any static parameters~\cite{badkobeh2014unbiased,bottcher2010optimal,doerr2015black,doerr2020optimal,doerr2021runtime}. This topic, which has always been studied in continuous optimization, has taken longer to gain traction in discrete domains~\cite{karafotias2014parameter,aleti2016systematic,doerr2015optimal,doerr2015black} but has attracted increasing interest over the last years~\cite{lassig2011adaptive,doerr2019selfadjusting,rajabi2020self,doerr2021self,rodionova2019offspring,lissovoi2019time,lissovoi2020duration,lissovoi2020simple,doerr2018runtime,doerr2018static,hevia2020choice,mambrini2015design,case2020self,rajabi2020evolutionary}. Instead of a detailed discussion we refer the reader to the book chapter~\cite{doerr2020theory} for an overview over theoretical results, and to~\cite{hevia2021arxiv} for a discussion of some recent developments.

One of the most traditional and influential methods for parameter control is the $(1:s+1)$-success rule~\cite{kern2004learning}, independently developed several times~\cite{rechenberg1978evolutionsstrategien,devroye72compound,schumer1968adaptive} and traditionally used with $s=4$ as one-fifth rule in continuous domains, e.g.~\cite{auger2009benchmarking}. This rule has been used for controlling the offspring population size in discrete domains~\cite{doerr2015black,doerr2015optimal}, in particular for the \oclea ~\cite{hevia2021arxiv,hevia2021self}, where it yields the so-called \emph{self-adjusting \oclea} or \emph{\saolea}, also called \saoleaparam. As in the basic \oclea, in each generation the algorithm produces $\lambda$ offspring, and selects the fittest of them as the unique parent for the  next generation. The difference to the basic \oclea is that the parameter $\lambda$ is replaced by $\lambda/F$ if the fittest offspring is fitter than the parent, and by $\lambda\cdot F^{1/s}$ otherwise. We will give a more thorough discussion of this algorithm in Section~\ref{sec:prelim} below. Thus, the $(1:s+1)$-success rule replaces the parameter $\lambda$ by two parameters~$s$ and~$F$. As outlined above, there are two hopes associated with this scheme:
\begin{enumerate}[(i)]
    \item that the performance is \emph{more robust} with respect to $F$ and $s$ than with respect to $\lambda$;
    \item that the scheme can adaptively find the \emph{locally optimal} value of $\lambda$ throughout the course of optimization.
\end{enumerate}
Recently, Hevia Fajardo and Sudholt have investigated both hypotheses on the \OneMax benchmark~\cite{hevia2021self,hevia2021arxiv}. They found a negative result for (i), and a (partial) positive result for (ii). The negative result says that performance is at least as fragile with respect to the parameters as before: if $s<1$, then the \saolea finds the optimum of \OneMax in $O(n)$ generations, but if $s\ge 18$ and $F\le 1.5$ the runtime becomes exponential with overwhelming probability. Experimentally, they find that the range of bad parameter values even seems to include the standard choice $s=4$, which corresponds to the $1/5$-rule. On the other hand, they show that for $s<1$, the algorithm successfully achieves (ii): they show that the expected number of function evaluations is $O(n\log n)$, which is optimal among all unary unbiased black-box algorithms~\cite{doerr2020optimal,lehre2012black}. Moreover, they show that the algorithm makes steady progress over the course of optimization, needing $O(b-a)$ generations to increase the fitness from $a$ to $b$ whenever $b-a \ge C\log n$ for a suitable constant $C$. The crucial point is that this is independent of $a$ and $b$, so independent of the current state of the algorithm. It implies that the algorithm chooses $\lambda = O(1)$ in early stages when progress is easy, and (almost) linear values $\lambda = \Omega(n)$ in the end when progress is hard. Thus, it achieves (ii) conditional on having appropriate parameter settings.

Interestingly, it is shown in~\cite{hevia2021arxiv} that for $s\ge 18$, the \saolea fails in a region far away from the optimum, more precisely in the region with $~85\%$ one-bits. Consequently, it also fails for every other function that is identical with \OneMax in the range of $[0.84n,0.85n]$ one-bits, which includes other classical benchmarks like \jump, \cliff, and \ridge. It is implicit that the algorithm would be efficient in regions that are closer to the optimum. This is remarkable, since usually optimization is harder close to the optimum. Such a reversed failure profile has previously only been observed in very few situations. One is the $(\mu+1)$-EA with mutation rate $c/n$ for an arbitrary constant $c>0$ on certain monotone functions. This algorithm is efficient close to the optimum, but fails to cross some region in linear distance of the optimum if $\mu > \mu_0$ for some $\mu_0$ that depends on~$c$~\cite{lengler2021exponential}. A similar phenomenon has been shown for $\mu =2$ and a specific value of $c$ in the dynamic environment \dynBV~\cite{lengler2020large,lengler2021runtime}. These are the only examples for this phenomenon that the authors are aware of.

A limitation of~\cite{hevia2021arxiv} is that it studies only a single benchmark, the \OneMax function. Although the negative result also holds for functions that are identical to \OneMax in some range, the agreement with \OneMax in this range must be perfect, and the positive result does not extend to other functions in such a way. This leaves the question on what happens for larger classes of benchmarks: 
\begin{enumerate}[(a)]
    \item Is there a safe choice for $s$ that makes the algorithm efficient for a whole class of functions?
    \item Does the positive result (ii) extend to other benchmarks than \OneMax?
\end{enumerate}
In this paper, we will answer both questions with \emph{Yes} for the set of all (strictly) monotone pseudo-Boolean functions, i.e., functions where flipping a zero-bit into a one-bit always increases the fitness.\footnote{Shortly after our work, Hevia Fajardo and Sudholt also provided such a class in~\cite{hevia2022hard}. This is the class of \emph{everywhere hard functions}, for which the chance of creating a strict improvement does not exceed $n^{-\eps}$ anywhere in the search space. This includes the popular \leadingones benchmark, but not \onemax. In fact, it does not include any monotone function, since from the all-zero string (and from all other strings with $\Omega(n)$ zero-bits) the probability of an improvement is $\Omega(1)$ on monotone functions. Hence, their class is disjoint from ours. In~\cite{hevia2022hard} it was shown that for any constant $s$ the \saoclea imitates the elitist \saoplea on everywhere hard functions, which by design can never lose fitness. This arguably makes the comma variant a bit pointless for everywhere hard functions, since its potential benefit of escaping from local optima~\cite{jorritsma2023comma} is suppressed in this case.} This is a very large class; for example, it contains all linear functions. In fact, all our results hold in an even more general \emph{dynamic} setting: the fitness function may be different in each generation, as long as it is a monotone function every time and therefore shares the same global optimum \((1, \ldots, 1)\). We show an upper bound of $O(n)$ generations (Theorem~\ref{thm:generations}) and $O(n\log n)$ function evaluations (Theorem~\ref{thm:num_evaluations_csmall}) if the mutation rate is $c/n$ for some $c<1$, which is a very natural assumption for monotone functions as many algorithms become inefficient for large values of $c$~\cite{doerr2013mutation,lengler2018drift,lengler2019does,lengler2019general}. Those results are as strong as the positive results in~\cite{hevia2021arxiv}, except that we replace the constant~``$1$'' in the condition $s<1$ by a different constant that may depend on $c$. For~$c=1$ we still show that a bound of $O(n\log n)$ generations (Theorem~\ref{thm:generations_c1}) and $O(n^2\log \log n)$ evaluations (Theorem~\ref{thm:num_evaluations_cone}). It is in line with general frameworks for elitist algorithms that the number of function evaluations stops being quasi-linear~\cite{jansen2007brittleness,doerr2013mutation}, although the bounds in other contexts are better than quadratic~\cite{colin2014monotonic}.

Both parts of the answer are encouraging news for the \saolea. It means that, at least for this class of benchmarks, there is a universal parameter setting that works in all situations. This resembles the role of the mutation rate $c/n$ for the \ooea on monotone functions: If $c <1+\eps$, then the \ooea is efficient on all monotone functions~\cite{doerr2013mutation,lengler2018drift,lengler2019does}, and for~$c<1$ this is known for many other algorithms as well~\cite{lengler2019general}. On the other hand, the \moea is an example where such a safe parameter choice for~$c$ does not exist: for any $c>0$ there is $\mu$ such that the \moea with mutation rate $c/n$ needs super-polynomial time to find the optimum of some monotone functions.

We do not just strengthen the positive result, but we show that the negative result generalizes in a surprisingly strong sense, too: for any arbitra\-ry muta\-tion rate $c/n$ where $c<1$, if $s$ is sufficiently large, then the \saolea needs exponential time on \emph{every} monotone function, Theorem~\ref{thm:Inefficient}. Thus, the failure mode for large~$s$ is not specific to \OneMax. On the other hand, we also generalize the result (implicit in~\cite{hevia2021arxiv}) that the only hard region is in linear distance from the optimum: for any value of $s$, if the algorithm starts close enough to the optimum (but still in linear distance), then with high probability it optimizes every monotone function efficiently, Theorem~\ref{thm:large_s_efficient}. Finally, we complement the theoretical ana\-lysis with simulations in Section~\ref{sec:simulations}. These simulations show another interesting aspect: in a `middle region' of~$s$, it seems to depend on $F$ whether the algorithm is efficient or not. Thus, we conjecture that there does \emph{not} exist an efficiency threshold $s_0$ such that all parameters $s<s_0$ and $F>1$ are efficient, while all $s>s_0$ and $F>1$ are inefficient. Note that our results show that this dependency on $F$ can only appear in a `middle region', since our results for small $s$ and large $s$ are independent of~$F$ (which improves the negative result in~\cite{hevia2021arxiv}). A similar effect was observed for the self-adjusting $(1+(\lambda,\lambda))$-GA in~\cite{doerr2018optimal}, but for different reasons. There, the effect was caused by a universal upper bound on the success probability of $\approx 0.31 < 1$, independent of $\lambda$. This can cause problems if the target success rate is larger than $0.31$, and thus unachievable, see \cite[Section 6.4]{doerr2018optimal} for a full discussion. In our setting, the success probability approaches one as $\lambda$ grows, so the problem does not exist, and the reason for the impact of $F$ seems different.

Our proofs build on ideas from~\cite{hevia2021arxiv}. In particular, we use a potential function of the form $g(x^t,\lambda^t) = \ZM(x^t) + h(\lambda^t)$, where $\ZM(x^t)$ is the number of zero-bits in $x^t$ and $h(\lambda^t)$ is a penalty term for small values of $\lambda^t$. Similar decompositions have been used before~\cite{doerr2016provably}. The exact form of $h$ depends on the situation; sometimes it is very similar to the choices in~\cite{hevia2021arxiv} (positive result for $c<1$, negative result), but some cases are completely different (positive results for $c=1$ and close to the optimum). With these potential functions, we obtain a positive or negative drift and upper or lower bounds on the number of generations, depending on the situation. When translating the number of generations into the number of function evaluations, while some themes from~\cite{hevia2021arxiv} reappear (e.g., to consider the best-so-far $\ZM$ value), the overall argument is different. In particular, we do not use the ratchet argument from~\cite{hevia2021arxiv}, see Remark~\ref{remark:error} for a discussion of the reasons.

\subsection{Discussion of the \saolea}

Let us give a short explanation of the concept of the $(1:s+1)$-success rule (or $(1:s+1)$-rule for short). For given $\lambda$ and given position $x$ in the search space, the algorithm has some \emph{success probability} $p$, where success means that $f(y) > f(x)$ for the fittest of $\lambda$ offspring $y$ of $x$. For simplicity we will ignore the rounding effect coming from $\lambda\in \NN$ and will assume that~$p(\lambda) \le 1/(s+1)$ for $\lambda = 1$, and that $0<p<1$. The success probability~$p=p(\lambda)$ is obviously an increasing function in $\lambda$, since additional offspring can only increase the chances of finding an improvement. Moreover, it is strictly increasing due to $0<p<1$. Hence there is a value $\lambda^*$ such that $p(\lambda) < 1/(s+1)$ for $\lambda < \lambda^*$ and $p(\lambda) > 1/(s+1)$ for $\lambda > \lambda^*$. Now consider the potential $\log_F \lambda$. This potential decreases by $1$ with probability $p$ and increases by $1/s$ with probability $1-p$. So in expectation it changes by $-p + (1-p)/s = (1-(s+1)p)/s$. Hence, the expected change is positive if $\lambda < \lambda^*$ and negative if $\lambda > \lambda^*$. Therefore, $\lambda$ has a drift towards $\lambda^*$ from both sides (in a logarithmic scaling). So the rule implicitly has a \emph{target population size} $\lambda^*$, and this population size $\lambda^*$ corresponds to the \emph{target success rate}~$p=1/(s+1)$. 

Note that a drift towards $\lambda^*$ does not necessarily imply that $\lambda$ always stays close to $\lambda^*$. Firstly, $p$ depends on the current state $x$ of the algorithm, and might vary rapidly as the algorithm progresses (though this does not seem a very typical situation). In this case, the target value $\lambda^*$ also varies. Secondly, even if $\lambda^*$ remains constant, there may be random fluctuations around this value, see~\cite{kotzing2016concentration,lengler2020drift} for treatments on when drift towards a target guarantees concentration. However, we note that the $(1:s+1)$-rule for controlling $\lambda$ gives stronger guarantees than the same rule for controlling other parameters like step size or mutation rate. The difference is that other parameters do not necessarily influence $p$ in a monotone way, and therefore we cannot generally guarantee that there is a drift towards success probability $1/(s+1)$ when the $(1:s+1)$-rule is used to control them. Only when controlling $\lambda$ we are guaranteed a drift in the right direction.

\section{Preliminaries and Definitions}\label{sec:prelim}

Throughout the paper we will assume that $c>0$, $s>0$ and $F>1$ are constants independent of $n$ while $n\to \infty$. Note that $s$ need not be an integer. Our search space is always $\{0,1\}^n$, and we denote by $\supp\{x\} := \{i\in [n] \mid x_i = 1\}$ the \emph{support} of a bit string $x\in \{0,1\}^n$. We say that an event $\mathcal E = \mathcal{E}(n)$ holds \emph{with high probability} or \emph{whp} if $\Pr[\mathcal E] \to 1$ for $n\to\infty$. We denote the negation of an event $\mathcal E$ by $\overline{\mathcal E}$, and by $\indicator{\mathcal E}$ the indicator of $\mathcal E$, i.e., $\indicator{\mathcal E} = 1$ if $\mathcal E$ holds and $\indicator{\mathcal E} = 0$ otherwise. 

\subsection{The Algorithm: \saolea}

We will consider the self-adjusting \oclea with $(1:s+1)$-success rate, with mutation rate $c/n$, success ratio $s$ and update strength $F$, and we denote this algorithm by \saolea. It is given by the following pseudocode. Note that the parameter $\lambda$ may take non-integral values during the execution of the algorithm, however the number of children generated at each step is chosen to be the closest integer \(\round{\lambda}\) to \(\lambda\). 

\begin{algorithm}
\caption{\saolea with success rate $s$, update strength $F$ and mutation rate $c/n$ for maximizing a fitness function $f:\{0,1\}^n \to \RR$.}
\SetKwInput{Init}{Initialization}
\SetKwInput{Mut}{Mutation}
\SetKwInput{Sel}{Selection}
\SetKwInput{Opt}{Optimization}
\SetKwInput{Upd}{Update}
\label{alg:saocl}
\Init{Choose $x^0 \in \{0,1\}^n$ uniformly at random and $\lambda^0 := 1$\;}
\Opt{
    \For{$t = 0,1,\dots$}
	{
		\Mut{
			\For{$j \in \{1,\dots,\round{\lambda^{t}} \}$}
				{
				$y^{t,j}\leftarrow$ mutate$(x^t)$ by flipping each bit 
                    independently with prob. $c/n$\;
				}}
		\Sel{
			 Choose $y^t = \argmax_i f(y^{t,i})$, breaking ties randomly;
            }
		\Upd{
                \algorithmicif\ {$f(y^t) > f(x^t)$}\quad \algorithmicthen\ $\lambda^{t+1} \leftarrow \max\{1, \lambda^t/F\}$; \\
                \phantom{\textbf{Update:}} \, \algorithmicelse\ $\lambda^{t+1} \leftarrow F^{1/s}\lambda^t$;\\
                \phantom{\textbf{Update:}} \, $x^{t+1}\leftarrow y^t$;
		}
        
    }
}
\end{algorithm}
We will often omit the index $t$ if it is clear from the context.

\subsection{The Benchmark: Dynamic Monotone Functions}

Whenever we speak of ``monotone'' functions in this paper, we mean strictly monotone pseudo-Boolean functions, defined as follows.
\begin{definition}
We call \(f: \{0,1\}^n \to \mathbb{R} \) \emph{monotone} if \(f(x) > f(y)\) for every pair \(x, y \in \{0,1\}^n\) with \(x \neq y\) and \(x_i \ge y_i\) for all \(1 \le i \le n\).
\end{definition}
In this paper we will consider the following set of benchmarks. For each~$t\in \NN$, let $f^t:\{0,1\}^n \to \mathbb{R}$ be a monotone function that may
change at each step depending on \(x^t\).
Then the selection step in the $t$-th generation of Algorithm~\ref{alg:saocl} is performed with respect to $f^t$. By slight abuse of notation we will still speak of \emph{a} dynamic monotone function $f$.

All our results (positive and negative) hold in this dynamic setup.
This set of benchmarks is quite general. Of course, it contains the static setup in which we only have a single monotone function to optimize, which includes linear functions and \onemax as special cases. It also contains the setup of Dynamic Linear Functions (originally introduced as Noisy Linear Functions in~\cite{lengler2018noisy}) and Dynamic BinVal~\cite{lengler2020large,lengler2021runtime}. On the other hand, all monotone functions share the same global optimum $(1\ldots 1)$, have no local optima, and flipping a zero-bit into a one-bit strictly improves the fitness. In the dynamic setup, these properties still hold ``locally'', within each selection step. Thus, the setup falls into the general framework by Jansen~\cite{jansen2007brittleness}, which was extended to the \emph{partially ordered EA} (PO-EA) by Colin, Doerr, F\'erey~\cite{colin2014monotonic}. This implies that the \ooea with mutation rate $c/n$ finds the optimum of every such Dynamic Monotone Function in expected time $O(n\log n)$ if $c<1$, and in time $O(n^{3/2})$ if $c=1$.

\subsection{Drift Analysis and Potential Functions}
Drift analysis is a key instrument in the theory of EAs. To apply it, one must define a \emph{potential function} and compute the expected change of this potential. A common potential for simple problems in EAs are the \onemax and \zeromax potential of the current state $x^t$, which assign to each search point $x\in \{0,1\}^n$ the number of one-bits and zero-bits, respectively:
\begin{align*}
    \OM(x) = \onemax(x)=\sum\nolimits_i x_i \ \text{ and } \ \ZM(x) = \zeromax(x)  = n-\OM(x).
\end{align*}
Note that for two bit strings \(x\)  and \(y\), \(\OM(|x-y|)\) computes their Hamming distance, where the difference and absolute value are taken component-wise. 

For our purposes, this potential function will not be sufficient since there is an intricate interplay between progress and the value of $\lambda$. Following~\cite{hevia2021self,hevia2021arxiv}, we use a composite potential function of the form $g(x^t,\lambda^t) = \ZM(x^t) + h(\lambda^t)$, where $h(\lambda^t)$ varies from application to application (Definitions~\ref{def:h_positive}, \ref{def:h_positive2}, \ref{def:h_positive3},~\ref{def:h_negative}). 
We will write $Z^t := \ZM(x^t)$, $H^t := h(\lambda^t)$ and $G^t := g(x^t,\lambda^t)$ throughout the paper.

Once the drift is established, the positive and negative statements about generations then follow from standard drift analysis~\cite{lengler2020drift}. In particular, we will use the Additive, Multiplicative, and Negative Drift Theorem, given below.\footnote{Theorem~\ref{thm:AdditiveDrift} and \ref{thm:MultiplicativeDrift} are stated in the references for \textit{finite} state spaces $S$ but continue to hold as long as $S$ is bounded, see the remark on p.6 in \cite{lengler2020drift}}
\begin{theorem}[Additive Drift Theorem~\cite{lengler2020drift}]
\label{thm:AdditiveDrift}
Let \((X^t)_{t \ge 0}\) be a sequence of non-negative random variables over a bounded state space \(S \subset \mathbb{R}_0^{+}\) containing the origin and let \(T:=\inf\{t \ge 0 \mid X^t = 0\}\) denote the hitting time of \(0\). Assume there exists \(\delta > 0\) such that for all \(t < T\), 
\[
    \EE \left[X^t-X^{t+1}\mid X^t \right] \ge \delta,
\]
then
\[
    \erw{[T \mid X^0]} \le \frac{X^0}{\delta}.
\]
\end{theorem}

\begin{theorem}[Multiplicative Drift Theorem~\cite{doerr2012multiplicative}]
\label{thm:MultiplicativeDrift}
Let \((X^t)_{t \ge 0}\) be a sequence of non-negative random variables over a bounded state space \(S \subset \mathbb{R}_0^{+}\) con\-taining the origin and such that $x_{\min} := \min\{x \in S: x > 0\}$ is well defined. Let \(T:=\inf\{t \ge 0 \mid X^t \le 0\}\) denote the hitting time of \(0\).
Suppose that there exists a constant \(\delta > 0\) such that for all \(t < T\),
\begin{equation*}
    \EE [X^{t} - X^{t+1} \mid X^{t}] \ge \delta X^{t}.
\end{equation*}
Then,
\begin{equation*}
    \EE [T \mid X^{0}] \le \frac{1 + \log (X^{0} / x_{\min}) }{\delta}.
\end{equation*}

\end{theorem}

\begin{theorem}[Negative Drift Theorem With Scaling \cite{oliveto2015improved}]
\label{thm:NegativeDriftWithScaling}
Let \((X^t)_{t \ge 0}\) be a sequence of random variables over a state space
$S \subset \mathbb{R}$. Suppose there exists an interval $[a, b] \subseteq \mathbb{R}$ and, possibly depending on $\ell := b-a$, a drift bound~$\varepsilon := \varepsilon(\ell) > 0$ as well as a scaling factor $r := r(\ell)$ such that for all~$t \ge 0$ the following three conditions hold:
\begin{enumerate}
    \item $\erw{[X^{t+1} - X^t \mid X^0, \dots, X^t; a < X_t < b}]\ge \varepsilon$.
    \item $\Pr[{\lvert X^{t+1} - X^t\rvert \ge jr \mid X_0, \dots, X^t; a < X^t}] \le e^{-j}$ for $j \in \mathbb{N}_0$.
    \item $1 \le r^2 \le \varepsilon \ell/(132\log(r/\varepsilon))$.
\end{enumerate}
Then for all $X^0 \ge b$, the first time ${T^* := \min\{t \ge 0 \colon X^t < a} \mid X^0, \dots, X^t\}$ when \(X\) drops below \(a\) satisfies
\begin{align*}
    \Pr[{T^* \le e^{\varepsilon \ell/(132r^2)}}] = O(e^{-\varepsilon \ell/(132r^2)}).
\end{align*}
\end{theorem}

\subsection{Concentration of Hitting Times}\label{sec:concentration}

In our analysis, we will prove concentration of the number of steps needed to improve the number of \(1\)-bits. We will use the following results of K\"otzing~\cite{kotzing2016concentration}, which we slightly reformulate for convenience.

\begin{definition}[Sub-Gaussian~\cite{kotzing2016concentration}]
    \label{def:subgaussian}
    Let \(\left( X^t \right)_{t\ge0}\) be a sequence of random variables and \(\cF = \left( \cF^t \right)_{t\ge0}\) an adapted filtration. We say that \(\left( X^t \right)_{t\ge0}\) is~\((\gamma, \delta)\)-sub-Gaussian if 
    \begin{align*}
        \EE \left[ e^{ z X^{t+1} } \mid \mathcal{F}^t \right] \le e^{\gamma z^2/2},
    \end{align*}
    for all \(t\) and all \(z \in [0, \delta]\).
\end{definition}

\begin{theorem}[Tail Bounds Imply Sub-Gaussian~\cite{kotzing2016concentration}]
    \label{thm:subexponential_implies_subgaussian}
    For every \(0 < \alpha, 0 < \beta < 1\), there exists \(\gamma, \delta >0\) such that the following holds. 
    
    Let \(\left( X^t \right)_{t \ge 0}\) be a random sequence and \(\mathcal{F}\) an adapted filtration. Assume that \(\EE [ X^{t+1} \mid \mathcal{F}^t ] \le 0\) and for all times \(t\) and that for all \(x\ge 0\) we have
    \begin{align*}
        \Pr\left[ |X^{t+1}| \ge x \mid \mathcal{F}^t \right] \le \frac{\alpha}{(1 + \beta)^x}.
    \end{align*}
    Then \(\left( X^t \right)_{t \ge 0}\) is \((\gamma, \delta)\)-sub-Gaussian.
\end{theorem}

\begin{theorem}[Concentration of Hitting Times~\cite{kotzing2016concentration}]
    \label{thm:subgaussian_hit_fast}
    For every \(\gamma, \delta, \varepsilon > 0\) there exists a \(D > 0\) such that the following holds. Let \(\left( X^t \right)_{t \ge 0}\) be a random sequence and \( \mathcal{F}\) an adapted filtration satisfying the following properties
    \begin{enumerate}[(i)]
        \item \(\EE [ X^{t+1} - X^{t} \mid \mathcal{F}^t ] \ge \varepsilon\).
        \item \(\left(\varepsilon - X^t \right)_{t \ge 0}\) is \((\gamma, \delta)\)-sub-Gaussian;
    \end{enumerate}
    Let \(T\) denote the first point in time when \(\sum_{t=1}^{T}{X^t} \ge N\), then for all \(\tau \ge 2N /\varepsilon\),
    \begin{align*}
        \Pr[ T > \tau ] \le \exp \left( - D \tau \right).
    \end{align*}
\end{theorem}

To prove concentration of the number of steps spent improving the fitness under multiplicative drift, we will use the following theorem. 
\begin{theorem}[Multiplicative Drift, Tail Bound~\cite{doerr2010drift}]
\label{thm:multiplicative_drift_concentration}
Let \((X^t)_{t \ge 0}\) be non-negative random variables over a state space \(S \subset \mathbb{R}_0^{+}\).
Assume that \(X^{0} \le b\) and let $T$ be the random variable that denotes the first point in time $t \in \mathbb{N}$ for which $X^{t} \le a$, for some $a \le b$. Suppose that there exists \(\delta > 0\) such that for all \(t < T\),
\begin{equation*}
    \EE [X^{t} - X^{t+1} \mid X^{t}] \ge \delta X^{t}.
\end{equation*}
Then,
\begin{align*}
    \Pr \left[ T > \frac{t + \log (b / a) }{\delta} \right] \le e^{-t}.
\end{align*}
\end{theorem}

\subsection{Further Tools}
We will use the FKG inequality (Fortuin–Kasteleyn–Ginibre inequality), which is a standard tool in percolation theory, but less commonly used in the theory of EAs. We only give a special case of what is known as \emph{Harris inequality}.
\begin{theorem}[FKG inequality~{\cite[Section~2.2]{grimmett1999percolation}}] \label{thm:FKG}
Let $I$ be a finite set, and consider a product probability space $\Omega = \prod_{i\in I} \Omega_i$, where all $\Omega_i$ have binary sample space $\{0,1\}$. A real-valued random variable $X$ is called \emph{increasing} if~$X(\omega) \le X(\omega')$ holds for all elementary events $\omega, \omega'$ in $\Omega$ with $\omega_i \le \omega_i'$ for all $i \in I$. It is called \emph{decreasing} if $-X$ is increasing. 
\begin{enumerate}
    \item If two random variables $X,Y$ are both increasing or are both decreasing, then
\begin{align*}
\EE[XY] \ge \EE[X]\cdot\EE[Y].
\end{align*}
\item If $X$ is increasing and $Y$ is decreasing, or vice versa, then
\begin{align*}
\EE[XY] \le \EE[X]\cdot\EE[Y].
\end{align*}
\end{enumerate}
\end{theorem}
We also say that $X$ and $Y$ are \emph{positively correlated} and \emph{negatively corre\-lated} in the first and second case respectively. Note that the FKG inequality also applies to probabilities. E.g., if $A,B$ are increasing events (which just means that their indicators $\indicator{A}$ and $\indicator{B}$ are increas\-ing), then $\Pr[A] = \EE[\indicator{A}]$, $\Pr[B] = \EE[\indicator{B}]$, and $\indicator{A}\indicator{B} = \indicator{A \cap B}$, hence $\Pr[A \cap B] \geq \Pr[A]\cdot\Pr[B]$.

To switch between differences and exponentials, we will frequently make use of the following estimates, taken from Lemma~1.4.2 -- Corollary~1.4.6 in~\cite{doerr2020probabilistic}. 
\begin{lemma}\label{lem:basic}
\leavevmode
\begin{enumerate}
    \item For all $r\ge 1$ and $0\le s\le r$,
\begin{align*}
    (1-1/r)^{r} \le 1/e \le (1-1/r)^{r-1}\end{align*}
    and 
\begin{align*}   (1-s/r)^{r} \le e^{-s} \le (1-s/r)^{r-s}.
\end{align*}
\item For all $0\le x \le 1$,
\begin{align*}
1-e^{-x} \ge x/2.
\end{align*}
\item For all $0 \le x\le 1$ and all $y>0$,
\begin{align*}
1- (1-x)^y \le \tfrac{xy}{1+xy}.
\end{align*}
\end{enumerate}
\end{lemma}

Finally, we will also use standard Chernoff bounds.
\begin{theorem}[Chernoff Bound~{\cite[Section~1.10]{doerr2020probabilistic}}] \label{thm:Chernoff}
Let $X_1, \ldots, X_n$ be in\-dependent random variables taking values in $[0,1]$. Let $X = \sum_{i = 1}^n X_i$ and let $1\ge \delta \ge 0$. Then 
\begin{align}
  \Pr[X \ge (1+\delta) E[X]] \le \exp\bigg(-\frac{\delta^2 E[X]}{3}\bigg).\label{eq:probCMUeasy}
  \end{align}
\end{theorem}

\section{Technical Definitions and Results}

As mentioned previously, our techniques resemble those of Hevia Fajardo and Sudholt~\cite{hevia2021self, hevia2021arxiv}. The key is analysing a suitable potential function \(g(x, \lambda) = \ZM(x) + h(\lambda)\) which combines the distance $\ZM(x)$ to the optimum (as defined in Section~\ref{sec:prelim}) with a penalty term for small $\lambda$. 
When this function has strong positive drift, we can establish that the optimum is reached fast; conversely, when \(g\) has (strong) negative drift, the optimisation takes super-polynomial time. 
In some cases, we use a similar penalty term \(h(\lambda)\) (and thus potential function) as~\cite{hevia2021arxiv}, in other cases very different ones. 
However, the potential always contains the number of zero-bits $Z^t = \ZM(x^t)$ at time~\(t\) as an additive term, so the drift of $Z$ enters the drift of the potential in all cases. The goal of this section is to compute this drift in Lemma~\ref{lem:drift_Z}. Moreover, the definitions and results for showing this lemma are also used at other places in the paper. 

Finally, we show that if the penalty term \(h(\lambda)\) is `reasonable', then the truncated change of \(g\) at each step $\min\{C,G^t-G^{t+1}\}$ has exponential tail bounds and is thus sub-Gaussian. This allows us to apply the concentration results from Section~\ref{sec:concentration} to establish concentration from above of the optimi\-sation time; see Remark~\ref{rem:truncation} for more details.

\subsection{Definition and Properties of Basic Events}\label{sec:basic}

Because our analysis deals with an entire class of functions, we will not be able to precisely compute the probability of finding a fitness improvement. However, since we study (dynamic) monotone functions we can relate that probability to the probability of 1) having a child that flips no \(1\)-bit of the parent, and 2) having a child that flips at least a \(0\)-bit of the parent. Understanding those two events, which we respectively denote by \(A\) and~\(B\) and formally define below, is the backbone of our approach and this sub\-section is devoted to their analysis. 

Recall that $x^t$ and $\lambda^t$ are the search point and the offspring population size at time $t$ respectively, and that $y^{t,j}$ denotes the $j$-th offspring at time $t$. For all times \(t\) we define 
\begin{align*}
    A^{t,j} := \{\supp(x^t) \subseteq \supp(y^{t,j})\} \quad \text{and} \quad A^t = \bigcup_j A^{t,j}.
\end{align*}
In words, \(A^{t,j}\) is the event that the \(j\)-th offspring at time \(t\) does not flip any one-bit of the parent, and \(A^t\) is the event that such a child exists at time \(t\). We also define 
\begin{align*}
    B^{t,j} = \{ \exists i \colon x^t_i < y^{t,j}_i\} \quad \text{and} \quad B^{t} = \bigcup_j B^{t,j},
\end{align*}
respectively as the event that the $j$-th child \emph{does} flip a \emph{zero}-bit of the parent, and the event that such a child exists. We drop the superscript $t$
when the time is clear from context, and just write $x,y^j$ and $\lambda$ for parent, offspring, and population size at time $t$, and $A^j,A, B^j,B$ for the events defined above. We also observe that all the events $\{A^j\}_j\cup\{B^j\}_j$ are independent and in particular $A$ and $B$ are independent.

In the lemma below we estimate the probability of \(A^t\) and \(B^t\) in terms of \(Z^t\) and \(\lambda^t\). We also provide a bound on the probability of \emph{not} finding a fitness improvement. 

\begin{lemma}
    \label{lem:properties_AB}
    For any mutation rate \(c \le 1\), 
    there exist constants \(b_1, b_2, b_3 >0\) depending only on \(c\) such that at all times \(t\) with $Z^t \ge 1$ we have 
    \begin{align*}
        \Pr[\bar{A}] \le e^{-b_1 \lambda} \qquad \text{and} \qquad e^{-b_2 \lambda Z^t/n} \le \Pr[\bar{B}] \le e^{-b_3 \lambda Z^t/n}.
    \end{align*}
    Moreover, $\Pr[f(x^{t+1}) \le f(x^t)] \le e^{-\tfrac12 ce^{-c}\lfloor \lambda \rceil Z^t/n}$.
\end{lemma}

\begin{proof}
    Let us start with the first inequality. The event \(A^j\) happens with probability \((1 - c/n)^{n - Z^t} \ge e^{-c}\) by Lemma~\ref{lem:basic}, so \(\bar{A} = \bigcap_j \overline{A^j}\) has probability
    \begin{align*}
        \Pr[\bar{A}] \le \left(1 - e^{-c} \right)^{\round{\lambda}} \le e^{-e^{-c} \round{\lambda} },
    \end{align*}
    again using Lemma~\ref{lem:basic}. We conclude the first proof by observing that \(\round{\lambda} \ge \lambda / 1.5\).
    
    The event \(\bar{B}\) happens if none of the \(\round{\lambda}\) offspring flips a zero-bit of the parent. This happens with probability
    \begin{align*}
        \Pr[\bar{B}] = (1 - c/n)^{Z^t \round{\lambda}}.
    \end{align*} 
    The upper bound is obtained as above: \((1 - c/n) \le e^{-c/n}\) by Lemma~\ref{lem:basic} and \(\round{\lambda} \ge \lambda / 1.5\). 
    For the lower bound, we see that \((1 - c/n) \ge e^{-2c/n}\) for~$c/n \le 1/2$ by Lemma~\ref{lem:basic} and \(\round{\lambda} \le 2 \lambda\) so that 
    \begin{align*}
        \Pr[\bar{B}] \ge e^{ -4c Z^t \lambda / n}.
    \end{align*}
    For the last inequality, for every $j$ the event $A^j \cap B^j$ implies $f(y^j) > f(x)$, since an offspring $y^j$ is always an improvement if it is obtained by flipping a single zero-bit and no one-bit. Since all $A^j$ and $B^j$ are independent and $\Pr[B^j] = 1-(1-c/n)^{Z^t} \ge 1-e^{-cZ^t/n} \ge \tfrac12 cZ^t/n$ by Lemma~\ref{lem:basic}, 
    \begin{align*}
        \Pr[f(x^{t+1}) \le f(x^t)] & = \Pr[\forall j: f(y^{t,j}) \le f(x^t)] 
        \le \prod_{j=1}^{\lfloor \lambda \rceil} \left(1-\Pr[A^j]\cdot \Pr[B^j]\right)  \\
        & \le \left(1-e^{-c}\cdot \tfrac12cZ^t/n\right)^{\lfloor \lambda \rceil}
        \le e^{-\tfrac12ce^{-c}\lfloor \lambda \rceil Z^t/n}.
    \end{align*} 
    
\end{proof}

\subsection{The Drift of \(Z^t\)}

With the definitions introduced in the previous subsection, we may now state and prove the key result of this section, that is, we compute the drift of \(Z\) in terms of \(\Pr[A], \Pr[B], Z\) and \(\lambda\).

\begin{lemma}
    \label{lem:drift_Z}
    Consider the \saolea with mutation rate \(0 < c \le 1\). There exist constants \(a_1, a_2, b > 0\) depending only on \(c\) such that at all times~\(t\) with $Z^t >0$ we have
    \begin{align*}
        \EE\left[ Z^t - Z^{t+1} \mid x^t, \lambda^t \right] \ge \Pr[ B ] \cdot a_1 \left(1 - c(1 - Z^t/n) \right) - a_2 e^{-b \lambda^t}.
    \end{align*}
    This also holds if we replace \(Z^t-Z^{t+1}\) by \(\min \{ 1, Z^t-Z^{t+1} \}\).
\end{lemma}

\begin{remark}
    \label{rem:truncation}
    Theorems~\ref{thm:subexponential_implies_subgaussian} and~\ref{thm:subgaussian_hit_fast} which we use to prove concentration of hitting times require that the probability of having large jumps is small. This is not true in general: when we generate many children \(\lambda\) there is an increased probability of flipping many bits. 

    In order to still be able to prove concentration, we consider the situation in which the number of \(0\)-bits may decrease by at most \(1\) at each step, i.e., this is why we cap the difference \(Z^t - Z^{t+1}\) at \(1\). Even under this pessimistic assumption, we prove (in Sections~\ref{sec:efficient} and~\ref{sec:efficient_near_optimum}) that  the drift is positive and the optimum is reached fast.
\end{remark}

The proof of this will be obtained using the following claims.

\begin{sclaim}
    \label{clm:conditional_drift_1}
    At all times \(t \ge 0\) with $Z^t > 0$ we have 
    \begin{align*}
        \EE \left[ Z^t-Z^{t+1} \mid A, B \right] \ge e^{-c}\left(1 - c \left( 1 - \tfrac{Z^t}{n} \right) \right).
    \end{align*}
    This also holds if we replace \(Z^t-Z^{t+1}\) by \(\min \{ 1, Z^t-Z^{t+1} \}\).
\end{sclaim}

\begin{sclaim}
    \label{clm:conditional_drift_2}
    At all times \(t \ge 0\) with $Z^t >0$ we have 
    \begin{align*}
        \EE \left[ Z^t-Z^{t+1} \mid \bar{A} \right] \ge -\frac{c}{1 - e^{-c}}.
    \end{align*}
    This also holds if we replace \(Z^t-Z^{t+1}\) by \(\min \{ 1, Z^t-Z^{t+1} \}\).
\end{sclaim}

\begin{proof}[Proof of Claim~\ref{clm:conditional_drift_1}]
    First, let us define \(K\) to be the index of the fittest offspring, i.e.\ \(y^{K} = x^{t+1}\). A first step in proving the claim will be to show that
    \begin{align}
        \Pr\left[ y^{K}_i = 0 \mid A, B^{K} \right] \le c/n, \label{eq:conditional_drift_1_intermediate}
    \end{align}
    for all \(i \in \supp(x)\). Note that~\eqref{eq:conditional_drift_1_intermediate} would hold with equality if we replaced $K$ by a fixed $j\in [\lfloor\lambda\rceil]$ and omitted conditioning on $A$, so the task is to show that conditioning on $A$ and conditioning on the offspring being selected can only decrease the probability. To show this, we use a multiple exposure of the randomness: we let \(u^{1}, \ldots, u^{\round{\lambda}}\) respectively be obtained from $x$ by only revealing the flips (or non-flips) of the \(0\)-bits of \(x\) in each of the \(\round{\lambda}\) children (where we abbreviate $x=x^t$ and $\lambda=\lambda^t$). The child \(y^{j}\), \(j\in [ \round{\lambda}]\) may then be obtained from $u^j$ by revealing the rest of the bits, i.e.\ the flips of the~\(1\)-bits of \(x\).
    
    Consider an index \(i\) such that \(x_i = 1\), and decompose 
    \begin{align*}
        &\Pr\left[ y_i^{K} = 0, A, B^{K} \mid ( u^{\ell} )_{\ell \in [ \round{\lambda}]} \right] \\ 
            & \qquad \qquad \qquad \qquad= \sum_{j=1}^{\round{\lambda}}{ \Pr\left[ K = j, y_i^{j} = 0, A, B^{j} \mid ( u^{\ell} )_{\ell \in [ \round{\lambda}]} \right]} \\
            & \qquad \qquad \qquad \qquad= \sum_{j=1}^{\round{\lambda}}{ \indicator{B^{j}} \cdot \Pr\left[ K = j, y_i^{j} = 0, A \mid ( u^{\ell} )_{\ell \in [ \round{\lambda}]} \right]}. \numberthis \label{eq:conditional_drift_1_intermediate_2}
    \end{align*}
    For a given \(j\), we observe that if we additionally condition on the \(1\)-bit flips in other children \(( y^{\ell} )_{\ell \ne j}\), then \(\indicator{K = j, A}\) is a decreasing function of the \(1\)-bit flips in the \(j\)-th child, while \(\indicator{ y^{j}_i = 0 }\) is an increasing function of those flips. The FKG inequality, Theorem~\ref{thm:FKG}, thus gives 
    \begin{align*}
        &\Pr\left[ K = j, y_i^{j} = 0, A \mid ( u^{\ell} )_{\ell \in [ \round{\lambda}]}, ( y^{\ell}_{} )_{\ell \in \left[ \round{\lambda}\right], \ell \ne j} \right]\\
            & \qquad \qquad \qquad \le \Pr\left[ y_i^{j} = 0 \mid ( u^{\ell} ), ( y^{\ell} )_{\ell \ne j} \right] \cdot \Pr\left[ K = j, A \mid ( u^{\ell} ), ( y^{\ell} )_{\ell \ne j} \right] \\
            & \qquad \qquad \qquad = \frac{c}{n} \cdot \Pr\left[ K = j, A \mid ( u^{\ell} ), ( y^{\ell} )_{\ell \ne j} \right],
    \end{align*}
    where the last line simply comes from the fact that the \(i\)-th bit flip in~\(y^{j}\) is independent of what happens in the other children and in the \(0\)-bit flips of~\(y^{j}\). 
    Using the law of total probability over \(( y^{\ell} )_{\ell \ne j}\) gives 
    \begin{align*}
        \Pr\left[ K = j, y_i^{j} = 0, A \mid ( u^{\ell} ) \right]
            &\le \frac{c}{n} \cdot \Pr\left[ K = j, A \mid ( u^{\ell} ) \right],
    \end{align*}
    and plugging this into \eqref{eq:conditional_drift_1_intermediate_2} gives
    \begin{align}\label{eq:FKG1}
        \Pr\left[ y^{K}_i = 0, A, B^{K} \mid ( u^{\ell} )_{\ell \in [ \round{\lambda}]} \right]
            &\le \sum_{j=1}^{\round{\lambda}}{ \indicator{B^{j}} \frac{c}{n} \Pr\left[ K = j, A \mid ( u^{\ell} )_{\ell \in [ \round{\lambda}]} \right]}.
    \end{align}
    A similar decomposition gives 
    \begin{align}\label{eq:FKG2}
    \Pr[A, B^{K} \mid ( u^{\ell} )_{\ell \in [ \round{\lambda}]} ] = \sum_{j=1}^{\round{\lambda}}{ \indicator{B^{j}} \cdot \Pr[ K = j, A \mid ( u^{\ell} )_{\ell \in [ \round{\lambda}]} ]},
    \end{align}
    and combining~\eqref{eq:FKG1} and~\eqref{eq:FKG2} gives \(\Pr[ y^{K}_i = 0 \mid A, B^{K}, ( u^{\ell} )_{\ell \in [ \round{\lambda}]} ]  = \eqref{eq:FKG1}/\eqref{eq:FKG2} \le c/n\). To obtain~\eqref{eq:conditional_drift_1_intermediate} it suffices to apply the law of total probability over~\(( u^{\ell} )_{\ell}\). We can now compute the drift conditioned on \(A, B^{K}\): since \(B^{K}\) implies that~\(y^K\) turns (at least) one \(0\)-bit into a \(1\)-bit, we obtain
    \begin{align}\label{eq:Claim1}
    \begin{split}
        \EE[ Z^t-Z^{t+1} \mid A, B^{K} ]
            &\ge 1 - \sum_{ i \in \supp(x) }{ \Pr[y^{K}_i = 0 \mid A, B^{K}] } \\
            &\stackrel{\eqref{eq:conditional_drift_1_intermediate}}{\ge} 1 - (n - Z^t) \cdot \frac{c}{n} = 1 - c\left( 1 - Z^t/n \right).
    \end{split}
    \end{align}
    Note that the first step in~\eqref{eq:Claim1} remains correct if we replace \(Z^t-Z^{t+1}\) with \(\min \left\{ 1, Z^t-Z^{t+1} \right\}\), and this difference does not play a role in any other parts of the proof. To continue the proof, we decompose
    \begin{align*}
        \EE[ Z^t-Z^{t+1} \mid A, B ]
            &= \Pr \left[ B^{K} \mid A, B \right] \cdot \EE\left[ Z^t-Z^{t+1} \mid A, B, B^{K} \right] \\
            & \qquad + \Pr \left[ \overline{B^{K}} \mid A, B \right] \cdot \EE\left[ Z^t-Z^{t+1} \mid A, B, \overline{B^{K}} \right].
    \end{align*}
    Since \(B \cap B^{K} = B^{K}\), the first conditional expectation of the RHS is exactly \( \EE[ Z^t-Z^{t+1} \mid A, B^{K} ] \ge 1 - c(1 - Z^t/n) \) by~\eqref{eq:Claim1}. To conclude the proof, it thus suffices to show 
    \begin{align}
        \Pr [ B^{K} \mid A, B ] \ge e^{-c}, \label{eq:conditional_drift_1_probability}
    \end{align} 
    and 
    \begin{align}
        \EE \left[ Z^t-Z^{t+1} \mid A, B, \overline{B^{K}} \right] = 0. \label{eq:conditional_drift_1_conditional_expectation}
    \end{align}
    
    We start by proving~\eqref{eq:conditional_drift_1_conditional_expectation}: we will argue that if both \(A\) and \(\overline{B^{K}}\) hold, then~\(x^{t+1} = x^{t}\). Indeed, \(\overline{B^{K}}\) implies that \(f(x^{t+1}) \le f(x^{t})\) since no bit is flipped from \(0\) to \(1\). Additionally, the equality holds if and only if \(x^{t+1} = x^{t}\), since flipping any \(1\)-bit to \(0\) would decrease \(f\) by strict monotonicity. On the other hand, one easily observes that \(A\) implies \(f(x^{t}) \le f(x^{t+1})\). Consequently, \(A \cap \overline{B^{K}}\) implies that \(x^t = x^{t+1}\), which proves \eqref{eq:conditional_drift_1_conditional_expectation}.
    
    Finally, we prove \eqref{eq:conditional_drift_1_probability}. Once \(( u^{\ell} )_{\ell \in [ \round{\lambda} ]}\) is revealed, we can define \(J\) as the set of those indices \(j\) which maximise \(f(u^{j})\). We observe that if \(B\) and~\(\cup_{j \in J}A^{j}\) hold, then \(B^{K}\) also does. Indeed, if we let \(J' \subseteq J\) be the set of indices for which \(A^{j}\) holds, and if \(J' \ne \emptyset\), then the set of children which maximise \(f( y^{j} )\) is exactly \(J'\). 
    Hence we have 
    \begin{align*}
        \Pr \left[ B^{K} \mid A, B \right]
            &\ge \Pr \left[ \cup_{j \in J}A^{j} \mid A, B \right]
            \ge \left( 1 - c/n \right)^{n - Z^t} \ge e^{-c}.
    \end{align*}
    The second inequality is simply obtained by noting that, under the assump\-tion that \(B\) holds and since \(J\) is not empty, the probability of \(\cup_{j \in J}A^{j}\) is at least that of \(A^j\) for a single (arbitrary) \(j\) in \(J\). The event \(A^j\) has probability~\((1 - c/n)^{n - Z^t}\), is independent of \(B\) and positively correlated with \(A\). The last inequality follows from Lemma~\ref{lem:basic} since $Z^t \geq 1$.
\end{proof}

\begin{proof}[Proof of Claim~\ref{clm:conditional_drift_2}]
    Recall that $\bar{A}$ is the event that every offspring flips at least one one-bit. Let K be the index of the fittest child and let $N^j$ be the number of one-bits flipped in the $j$-th offspring, we want to show that $\EE[N^K \mid \bar{A}] \le \EE[N^j \mid \bar{A}]$ holds for all $j \in [\round{\lambda}]$, i.e. the fittest offspring does not flip more one-bits than an arbitrary offspring in expectation.
    
    Note that conditioning on $\bar{A}$ leads to dependent bit flips within each individual offspring, but once we know that a specific one-bit is flipped, the remaining one-bit flips are independent. Therefore, we can couple the one-bits flips given $\bar{A}$ with the following procedure. Assume there are $m$ one-bits in $x$, we first sample the position of the \emph{first} ($=$ left-most) one-bit $l$ to be flipped. Afterwards, we still flip each bit \emph{to the right of $l$} independently with probability $c/n$. This gives the usual distribution of one-bit flips, conditioned on $\bar A$. To make this formal, we sample $l \in [m]$ with probability $p_l(m) = (1-(1-c/n)^m)^{-1} (c/n) (1-c/n)^{l-1}$ and flip the $l$-th one-bit. It is easy to verify that $\sum_{l=1}^m p_l(m)= 1$ since $p_l$ is a geometric sequence. Then for each~$l' \in [m]\setminus[l]$, we flip the $l'$-th one-bit independently with probability~$c/n$. The probability that a specific one-bit is flipped given $\bar{A}$ is $(c/n) (1-(1-c/n)^m)^{-1}$. By our procedure, this probability is
    \begin{align*}
        \Pr[\text{the $l$-th one-bit is flipped}] & = p_l(m) + \sum_{i=1}^{l-1} p_{i}(m) \frac{c}{n} \\
        &= p_l(m) + \frac{p_1(m) (1-(1-c/n)^{l-1})}{1-(1-c/n)} \frac{c}{n} \\
        & = \frac{c}{n} \frac{1}{1-(1-c/n)^m},
    \end{align*}
    which is exactly the desired conditional probability. 
    
    Therefore, we can get rid of $\bar{A}$ as follows. Let $N^j_{a:b}$ be the random number of bit flips when we flip the $a$-th to the $b$-th one-bit independently with probability $c/n$ for offspring $j$ and $l^j$ be the index $l$ sampled for offspring $j$. Then
    \begin{align*}
        \EE[N^K \mid \bar{A}] &= \sum_{j=1}^{\round{\lambda}} \EE[N^j \cdot \indicator{K=j} \mid \bar{A}] \\
        & =  \sum_{j=1}^{\round{\lambda}} \sum_{i=1}^m \EE[(1+N^j_{i+1:m}) \cdot \indicator{K=j} \mid l^j= i] \Pr[l^j = i] \\
        &=  \sum_{j=1}^{\round{\lambda}} \sum_{i=1}^m \EE[(1+N^j_{i+1:m}) \cdot \indicator{K=j}] \cdot p_i(m).
    \end{align*}
    Now we may use the FKG inequality to show 
    \begin{align}
        \EE[ N^j_{i+1:m} \cdot \indicator{K=j} ] \le \EE[ N^j_{i+1:m}] \cdot \EE[ \indicator{K=j}]. \label{eq:conditional_drift_2_FKG_subtelty}
    \end{align}
    The proof is similar to that used in the previous claim: one conditions on~\((u^\ell),~(y^\ell)_{\ell \ne j}\) in order to have a product space on which \(N^j_{i+1:m}\) is increas\-ing and \(\indicator{K = j}\) decreasing. One then applies the FKG inequality. Observing that \(N^j_{i+1:m}\) is independent of \((u^\ell), (y^\ell)_{\ell \ne j}\) and using the law of total probabi\-lity over \((u^\ell), (y^\ell)_{\ell \ne j}\) gives \eqref{eq:conditional_drift_2_FKG_subtelty}. Continuing the previous derivation, 
    \begin{align*}
        \EE[N^K \mid \bar{A}] 
        &\le \sum_{j=1}^{\round{\lambda}} \sum_{i=1}^m \EE[1+N^j_{i+1:m}] \EE[ \indicator{K=j}] \cdot p_i(m) \\
        & = \sum_{j=1}^{\round{\lambda}} \EE[ \indicator{K=j}] \sum_{i=1}^m \EE[1+N^j_{i+1:m}] \cdot p_i(m) \\
        &= \sum_{i=1}^m \EE[1+N^j_{i+1:m} \mid l^j=i] \Pr[l^j=i] = \EE [N^j \mid \bar{A}],
    \end{align*}
    where we use the fact that $\sum_{j=1}^{\round{\lambda}} \indicator{K=j} = 1$ and that $\EE[(1+N^j_{i+1:m})]$ is invariant with respect to the index $j$.
    
    The term $\EE[N^j \mid \bar{A}]$ is maximized for $m=n$, hence
    \begin{align*}
        \EE[Z^t -Z^{t+1} \mid \bar{A}] &\ge - \EE[N^K \mid \bar{A}] \ge -\EE[N^j \mid \bar{A}] \\
        & \ge - \frac{c}{n} \frac{1}{1-(1-c/n)^n} n \ge - \frac{c}{1-e^{-c}},
    \end{align*}
    which proves Claim~\ref{clm:conditional_drift_2}.
\end{proof}

We now combine the two claims above to obtain Lemma~\ref{lem:drift_Z}.

\begin{proof}[Proof of Lemma~\ref{lem:drift_Z}]
    The drift of \(Z^t = \ZM(x^t)\) may be decomposed as follows,
    \begin{align}\label{eq:drift_of_G}
    \begin{split}
        \EE[ Z^t-Z^{t+1} \mid x, \lambda ]
            & = \Pr[A, B] \cdot \EE[ Z^t - Z^{t+1} \mid A, B ] \\
            &\qquad + \Pr[A, \bar{B}] \cdot \EE[ Z^t-Z^{t+1} \mid A, \bar{B} ] \\ 
            & \qquad + \Pr[\bar{A}] \cdot \EE[ Z^t-Z^{t+1} \mid \bar{A} ],
    \end{split}
    \end{align}
    where we omitted the conditioning on $x, \lambda$ on the right-hand side for brevity. As observed above, \(A, B\) are independent so we get \(\Pr[A, B] = \Pr[A] \Pr[B]\). Also, we observe that the second conditional expectation in \eqref{eq:drift_of_G} must be~\(0\): if~\(\bar{B}\) holds then no child is a strict improvement of the parent, but~\(A\) guarantees that some children are at least as good. Hence, if \(A, \bar{B}\) hold, we must have \(x^t = x^{t+1}\). Combining those remarks with the bounds of Claims~\ref{clm:conditional_drift_1} and \ref{clm:conditional_drift_2} gives 
    \begin{align*}
        \EE[ Z^t-Z^{t+1} \mid x,\lambda  ] \ge \Pr[A] \Pr[B] \cdot \frac{1 - c(1 - Z^t/n)}{e^c} - \Pr[\bar{A}] \cdot \frac{c}{1 - e^{-c}}. \label{eq:cor:drift_1}
    \end{align*}
    Lemma~\ref{lem:properties_AB} guarantees that \(\Pr[A]\) is at least a positive constant \(C\) when~\(\lambda \ge 1\) and that \(\Pr[\bar{A}] \le e^{-b_1 \lambda}\). Choosing \(a_1 = Ce^{-c}\), \(a_2 = c / (1 - e^{-c})\) and \(b = b_1\) gives the result.
    
\end{proof}

\subsection{Improvements are Sub-Gaussian}

In Sections~\ref{subsec:generations_csmall} and~\ref{sec:efficient_near_optimum} we will prove that the number of time steps needed to optimise a function is tightly concentrated. We provide the following result, based on Theorems~\ref{thm:subexponential_implies_subgaussian} and~\ref{thm:subgaussian_hit_fast}, which allows us to relate strong positive drift and concentration of hitting times. 

\begin{lemma}
    \label{lem:improvements_subgaussian}
    Consider the \saolea with parameters \(0 < c < 1 < F\) and with an arbitrary success rate \(s>0\), and let \(\varepsilon, C_1, C_2 > 0 \) be constants. Then there exist \(\gamma, \delta >0\) which only depend on \(c, F, s, \varepsilon, C_1, C_2\) such that the following holds.
     
    Let \(h\) be a decreasing function, \(g(x, \lambda) = \ZM(x) + h(\lambda)\) and define 
    \[
        \Gamma^t := 
        \min\{ 1 + C_1, g(x^t, \lambda^t) - g(x^{t+1}, \lambda^{t+1}) \}.
    \] 
    Assume the following two properties are satisfied 
    \begin{enumerate}[(i)]
        \item \(\EE\left[ \Gamma^{t} \mid x^t, \lambda^t \right] \ge \varepsilon\); \label{lem:itm:improvements_subgaussian_1}
        \item \(h(\lambda) - h(\lambda F) \le C_2\) for all \(\lambda \in [1, \infty)\). \label{lem:itm:improvements_subgaussian_2}
    \end{enumerate}
    Then \( \left( \varepsilon - \Gamma^t \right)_{t\ge0}\) is \((\gamma, \delta)\)-sub-Gaussian.

\end{lemma}

\begin{proof}
    We will use Theorem~\ref{thm:subexponential_implies_subgaussian} to prove that \((\varepsilon - \Gamma^t)\) is sub-Gaussian. To be able to apply this theorem, we must prove that \(\EE\left[ \varepsilon - \Gamma^t \right] \le 0\) and that there exist \(0 < \alpha\) and \(0 < \beta < 1\) such that \(\Pr[ | \varepsilon - \Gamma^t | > w ] \le \alpha / (1 + \beta)^w\) for all \(w \ge 0\). The first immediately holds by \ref{lem:itm:improvements_subgaussian_1}, so we focus on the second.
    
    Let \(y^{1}, \dots, y^{\lfloor{\lambda}\rceil}\) be the children at step \(t\) and let \(K = \argmax_j f(y^{j})\) the index of the fittest child. We also let \(N^{1}, \ldots, N^{\lfloor{\lambda}\rceil}\) be the number of \(1\)-bit flips in \(y^{1}, \ldots, y^{\lfloor{\lambda}\rceil}\) and we define \(N = N^{K}\). Clearly \(N \ge Z^{t+1} - Z^t\).
    
    Let \(w > C_2 + \varepsilon\) and \(w' = \lceil w - \varepsilon - C_2 \rceil\). Since the value of \(h\) may only decrease by \(C_2\) at each step by \ref{lem:itm:improvements_subgaussian_2}, to have \(\left( \varepsilon - \Gamma^t \right) \ge w\) it must be that~\(Z^{t+1} - Z^t \ge w - \varepsilon - C_2\), and in particular we must have \(N \ge w'\). 
    We compute
    \begin{align*}
        \Pr[ N \ge w' ]
            &\le \sum_{ j }{ \Pr[ N^{j} \ge w', K = j ] } \le \sum_{ j }{ \Pr[ N^{j} \ge  w'] \Pr[ K = j ] } \\
            &= \sum_{j}{ \binom{n - Z^t}{w'} \left(\frac{c}{n} \right)^{w'} \left( 1 - \frac{c}{n} \right)^{n - Z^t - w'} \Pr[ K = j ] } \\
            &\le 1/w'! \le 2^{-w'} \le 2^{\varepsilon + C_2 - w}.
    \end{align*}
    Above, the second inequality is obtained using the FKG inequality in the same fashion as in the proof of Claim~\ref{clm:conditional_drift_1}.
    
    The above implies that for all \(w > C_2 + \varepsilon\), the probability of having \(\varepsilon - \Gamma^t \ge w\) is bounded by \(\alpha / (1 + \beta)^w\) for \(\alpha = 2^{\varepsilon + C_2}\) and \(\beta = 1\). Up to possibly increasing \(\alpha\) to be a large constant, the same relation holds for~\(w \in [0, C_2 + \varepsilon]\). Since \(\Gamma^t\) is upper bounded by the constant $1 + C_1$, the quantity of interest \(\varepsilon - \Gamma^t\) is lower-bounded by \(\varepsilon - 1 - C_1\) so we can also trivially achieve 
    \begin{align*}
        \Pr[ \Gamma^t -\eps \ge w \mid x^t, \lambda^t ] \le \frac{\alpha}{(1 + \beta)^w},
    \end{align*}
    by possibly increasing \(\alpha\) again.
    
    Theorem~\ref{thm:subexponential_implies_subgaussian} may now be applied: \(\varepsilon - \Gamma^t\) is \((\gamma,\delta)\)-sub-Gaussian for some~\(\gamma, \delta\) depending on \(\varepsilon, C_1, C_2\).
    
\end{proof}


\section{Monotone Functions Are Efficiently Optimized for Small Suc\-cess Rates}
\label{sec:efficient}

In this section we analyse the \saolea when the success rate \(s\) is small, and the mutation rate is $c/n$ for a constant \(0 < c \le 1\). We show that if $s$ is sufficiently small then for \emph{any} strictly monotone fitness function, the optimum is found efficiently both in the number of generations and evaluations. We distinguish between the cases $c<1$ and $c=1$.

\subsection{Bound on the Number of Generations}
\label{subsec:efficient_gen}

In this subsection, we study the number of generations required to reach the optimum and show that for \(c \le 1\), the \saolea finds the optimum efficiently. We start with the case $c<1$. 

\begin{theorem}
    \label{thm:generations}
    Let \(0 < c < 1 < F\) be constants. Then there exist \(C, s_0 > 0\) such that for all \(0 < s \le s_0\) and for every dynamic monotone function the expected number of generations of the \saolea with success rate $s$, update strength $F$ and mutation probability \(c/n\) is at most  \(Cn\). 
\end{theorem}

For $c=1$, we additionally need to assume that the update strength~$F$ is bounded from above by a suitable constant $F_0 >1$. As we will show experimentally, the update strength can have a notable impact on perfor\-mance, but it remains open whether this effect vanishes for sufficiently small~$s$.

\begin{theorem}
    \label{thm:generations_c1}
    There exist constants $F_0>1$, $s_0>0$ and $C>0$ such that for all $1 < F < F_0$, all \(0 < s \le s_0\), and for all dynamic monotone functions the expected number of generations of the \saolea with success rate $s$, update strength $F$ and mutation probability \(1/n\) is at most  \(Cn \log n\).
\end{theorem}

Our approach will be essentially the same for both theorems and will follow the ideas of Hevia Fajardo and Sudholt~\cite{hevia2021arxiv,hevia2021self}. We prove them in the following two subsections.


\subsection{Expected Number of Generations When \(c < 1\)}
\label{subsec:generations_csmall}

We will prove Theorem~\ref{thm:generations} in this section; for the remainder of this section, we assume that \(0 < c < 1 < F\) and the dynamic monotone function~$f$ are all given and we will show the existence of a desired \(s_0\) independent of~\(f\). Recall that \(x^t\) is the search point at time \(t\), its children are $y^{t,1},\ldots,y^{t,\round{\lambda_t}}$, and \(\lambda^t\) is the value of \(\lambda\) at time \(t\). In particular, the latter does not need to be an integer, and the actual number of offspring at time~\(t\) is the closest integer \(\lfloor{\lambda^t }\rceil\). Whenever the time \(t\) is clear from the context, we will remove it from the superscript.

We show that for an appropriate function \(g\), the drift \(\EE [ g(x^t, \lambda^t) - g(x^{t+1}, \lambda^{t+1}) ]\) is positive. Our choice of \(g\) will guarantee that \(g(x, \cdot) = 0\) implies \(x = (1,\cdots, 1)\), and the Additive Drift Theorem~\ref{thm:AdditiveDrift} will allow us to bound the time until this happens. For this section, we use the following \(g = g_1\).
\begin{definition}[Potential function for positive result]\label{def:h_positive}
    Let 
    \[
        h_1(\lambda) := K_1 \cdot \max\{0, \log_F \frac{\lambda_{\max}}{\lambda}\},
    \] 
    where \(K_1\) is a constant to be chosen later, and \(\lambda_{\max} := F^{1/s} n\). Then for all~$x\in\{0,1\}^n$ and $\lambda \in [1,\infty)$ we define
    \begin{align*}
        g_1(x, \lambda) := \ZM(x) + h_1(\lambda),
    \end{align*}
    Recall that $Z^t = \ZM(x^t)$, and denote $H_1^t:= h_1(\lambda^t)$ and $G_1^t := g_1(x^t,\lambda^t)$.
\end{definition}

Our first lemma states that \(g_1(x, \lambda)\) does not deviate much from \(\ZM(x)\) for all \(x, \lambda\), and that it suffices to show that \(g_1\) reaches $0$, since then the optimum is found.

\begin{lemma}
[``Sandwich'' inequalities relating the potential function and the fitness]
    \label{lem:sandwich}
    For all \(x \in\{0,1\}^n\) and \(\lambda \in [1,\infty)\) we have 
    \begin{align*}
        g_1(x, \lambda) - K_1 \log_F \lambda_{\max} \le \ZM(x) \le g_1(x, \lambda).
    \end{align*}
    In particular \(g_1(x, \lambda) = 0\) implies that \(\ZM(x) = 0\).
\end{lemma}

\begin{proof}
    The lemma follows trivially from the fact that 
        \[h_1(\lambda) = K_1 \max \left\{ 0, \log_F \left( \lambda_{\max} / \lambda \right) \right\} \in \left[0 , K_1 \log \lambda_{\max} \right].\]
\end{proof}

We will now compute the drift of \(G_1^t\). 
The drift of \(Z^t\) was already computed in the previous section, so it suffices to to compute that of \(H_1^t\).

\begin{sclaim}
    \label{clm:small_sc_drift_h}
    At all times \(t \ge 0\) we have
    \begin{align*}
        \EE \left[ H_1^t-H_1^{t+1} \mid x^t,\lambda^t  \right] \ge - K_1 \cdot \Pr \left[B\right] + \frac{K_1}{s} \cdot \Pr \left[ \bar{B} \right] \cdot \indicator{\lambda_t < n}. 
    \end{align*}
\end{sclaim}

\begin{proof}[Proof of Claim~\ref{clm:small_sc_drift_h}]
    We first give a general bound that we will use for the case that the fitness increases. We have $\lambda^{t+1} \ge \lambda^t/F$ and thus $\log_F (\lambda_{\max}/\lambda^{t+1}) \le 1 + \log_F (\lambda_{\max}/\lambda^{t})$ and $H_1^{t+1} \le K_1 + H_1^t$. In particular,
    \begin{align*}
        \EE \left[ H_1^t-H_1^{t+1} \mid x^t, \lambda^t, f(x^{t+1}) > f(x^t) \right] \geq -K_1.
    \end{align*}
    For $f(x^{t+1}) \le f(x^t)$, we have $\lambda^{t+1}\le \lambda^t$, and thus $H_1^{t}-H_1^{t+1} \geq 0$. If additionally $\lambda^t < \lambda_{\max} F^{-1/s} = n$, then $\lambda^t < F^{1/s}\lambda^{t} = \lambda^{t+1} \le \lambda_{\max}$, and hence $H_1^{t}-H_1^{t+1} = K_1/s$. Summarizing,
    \begin{align*}
        \EE \left[ H_1^t-H_1^{t+1} \mid x^t, \lambda^t, f(x^{t+1}) \le f(x^t) \right]
            &\ge\begin{cases}
                    K_1/s & \text{if \(1 \le \lambda^t < n\);}\\
                    0 & \text{if \(\lambda^t \ge n\).}\\
                \end{cases}
    \end{align*}
    This gives 
    \begin{align*}
        \EE\left[ H_1^t-H_1^{t+1} \right] \ge  - K_1 \Pr[f(x^{t+1}) > f(x^t)] + \frac{K_1}{s}  \Pr[f(x^{t+1}) \le f(x^t )] \, \indicator{\lambda^t < n }.
    \end{align*}
    Observe that \(f(x^{t+1}) > f(x^t)\) only if \(B\) holds: indeed for the fitness to increase at time \(t\), at least one child needs to mutate a \(0\)-bit of \(x^t\) into a~\(1\)-bit. This gives \(\Pr[ f(x^{t+1}) > f(x^t) ] \le \Pr[B]\), and Claim~\ref{clm:small_sc_drift_h} follows.
\end{proof}

Claims~\ref{clm:conditional_drift_1}, \ref{clm:conditional_drift_2} and \ref{clm:small_sc_drift_h} may now be combined to obtain the following drift of \(G_1^t\). We again drop the index $t$ from $x^t$ and $\lambda^t$.

\begin{corollary}
    \label{cor:drift_G1}
    There exists a constant \(s_0 > 0\) such that for all \(0 < s \le s_0\) the following holds. There is a constant \(\delta\) and a choice of \(K_1\) such that for all $t$ with $Z^t >0$,
    \begin{align*}
        \EE\left[ G_1^t - G_1^{t+1} \mid x,\lambda \right] \ge \delta.
    \end{align*}
    This also holds if \(G^t - G^{t+1}\) is replaced by \(\min\{ 1 + K_1/s, G^t - G^{t+1} \}\).
\end{corollary}

\begin{proof}
    Combining Lemma~\ref{lem:drift_Z} and Claim~\ref{clm:small_sc_drift_h}, one obtains that the drift of \(G_1\) is at least 
    \begin{align}
        \EE\left[ G_1^t-G_1^{t+1} \mid x, \lambda \right]
            &\ge \Pr[B] \left( \alpha_1 - K_1 \right) + \indicator{\lambda < n} \Pr[\bar{B}] K_1/s - \alpha_2 e^{-\beta \lambda}, \label{eq:cor:drift_G1_initial}
    \end{align}
    for some constants \(\alpha_1, \alpha_2, \beta > 0\).
    
    We choose \(K_1 = \alpha_1 / 2\) so that the drift for any $t$ with $Z^t >0$ is at least 
    \begin{align*}
        \EE\left[ G_1^t-G_1^{t+1} \mid x, \lambda \right]
            &\ge \Pr[B] K_1 + \indicator{\lambda < n} \Pr[\bar{B}] K_1/s - \alpha_2 e^{-\beta \lambda}.
    \end{align*}
    
    \case{If \(\lambda < n\)} 
    then as we want \(s\) small enough, we may assume that \(s < 1\). In this setting the drift is lower bounded by
    \begin{align*}
        \EE[ G_1^t-G_1^{t+1} \mid x, \lambda ]
            &\ge \Pr[B] K_1 + \Pr[\bar{B}] K_1/s - \alpha_2 e^{-\beta \lambda} \\
            &= K_1 + \Pr[\bar{B}] K_1 (1-s)/s - \alpha_2 e^{-\beta \lambda}.
    \end{align*}
    Note that there is $\lambda_0 = \lambda_0(\alpha_2,\beta,K_1)$ such that for $\lambda \ge \lambda_0$ the last term can be bounded as $\alpha_2e^{-\beta\lambda} \le K_1/2$, in which case the drift is at least $K_1/2$. For the remaining case, recall Lemma~\ref{lem:properties_AB} guarantees that \(\Pr[\bar{B}] \ge e^{-b_2 \lambda}\) for some constant \(b_2 > 0\) depending only on \(c\). Hence, for a choice of \(s\) small enough we can achieve $\Pr[\bar{B}]K_1(1-s)/s \ge \alpha_2e^{-\beta\lambda}$, so the drift stays above~\(K_1 > K_1/2\).
    
    \case{If \(\lambda \ge n\)}
    the drift is 
    \begin{align*}
        \EE\left[ G_1^t - G_1^{t+1} \mid x, \lambda \right]
            &\ge \Pr[B] K_1 - \alpha_2 e^{-\beta n}.
    \end{align*}
    The first term is at least~\(\Pr[B] K_1 \ge ( 1 - (1 - c/n)^{\lfloor{\lambda}\rceil Z^t} ) K_1 \ge (1 - e^{-c}) K_1\) while the second is \(e^{-\Omega(n)} = o(1)\); this implies that the drift is at least \(\frac{(1 - e^{-c})}{2} K_1\) for sufficiently large $n$. 
    
    To see why the statement also holds for \(\min\{1 + K_1/s \; , \; G^t - G^{t+1}\}\), we recall that \(G^t = Z^t + H_1^t\), that Lemma~\ref{lem:drift_Z} holds for \(\min\{1, Z^t - Z^{t+1}\}\) and that~\(H_1\) may increase by at most \(K_1 / s\) in each step. This implies that the first formula \eqref{eq:cor:drift_G1_initial} also holds if we replace \(G^t - G^{t+1}\) by \(\min\{1 + K_1/s \; , \; G^t - G^{t+1}\}\) and all following arguments are unchanged. 
\end{proof}

We are now ready to prove the main theorem of this section.

\begin{proof}[Proof of Theorem~\ref{thm:generations}]
    Corollary~\ref{cor:drift_G1} guarantees that for \(s\) sufficiently small there is $\delta >0$ such that the drift of \(G_1\) is at least \(\EE [ G_1^t - G_1^{t+1} \mid x, \lambda ] \ge \delta \) whenever $Z^t >0$. 
    Let $T$ be the first point in time when either $G_1^T=0$ or~$Z^T=0$. Then the drift bound for $G_1$ applies to all $t<T$, and by Theorem~\ref{thm:AdditiveDrift} we have \(\EE[T] \le G_1^0/\delta \le (n+K_1\log(\lambda_{\max})) / \delta = O(n)\). By Lemma~\ref{lem:sandwich}, $G^T_1=0$ implies $Z^T=0$, so in particular at time \(T\) we have \(x^T = (1, \ldots, 1)\) and Theorem~\ref{thm:generations} is proved. 
\end{proof}


\subsection{Expected Number of Generations When \(c = 1\)}
\label{sec:gen_c1}

We will now prove Theorem~\ref{thm:generations_c1}; that is, we will show that the self-adjusting EA is also efficient when the mutation rate is \(1/n\). The reason we need to treat this case differently from the previous one is because of the expected number of bits gained when increasing the fitness. If we set~\(c = 1\), the drift obtained in Claim~\ref{clm:conditional_drift_1} is no longer constant but proportional to~\(Z^t/n\).
In particular, in the last stages of the exploration, the drift is a lot smaller and this results in a looser bound for the number of generations. 
Still, the proof is similar to the one for $c<1$, but we need to choose a different potential function.
\begin{definition}[Potential function for $c=1$]\label{def:h_positive2}
    Let 
    \[
        h_2(\lambda) :=  K_2 \cdot \max\{0, \tfrac{1}{\lambda} - \tfrac{1}{\lambda_{\max}}\},
    \] where \(K_2\) is a constant to be chosen later, and \(\lambda_{\max} = F^{1/s} n \). 
    Then for $x\in\{0,1\}^n$ and $\lambda \in [1,\infty)$ we define
    \begin{align*}
        g_2(x, \lambda) := \ZM(x) + h_2(\lambda),
    \end{align*}
    and we set $H_2^t := h_2(x^t,\lambda^t)$ and $G_2^t := g_2(x^t,\lambda^t)$, and as before $Z^t := \ZM(x^t)$.
\end{definition}

As before, we have a lemma stating that the deviation between $\ZM(x)$ and $g_2(x,\lambda)$ is small.
\begin{lemma}
[``Sandwich'' inequalities]
    \label{lem:sandwich_c1}
    For all \(x, \lambda\) we have 
    \begin{align*}
        g_2(x, \lambda) - K_2 \left(1-\frac{1}{\lambda_{\max}} \right) \le \ZM(x) \le g_2(x, \lambda).
    \end{align*}
    In particular \(g_2(x, \lambda) = 0\) implies that \(\ZM(X) = 0\).
\end{lemma}

\begin{proof}
    Similar to the proof of Lemma \ref{lem:sandwich}, the proof follows from the fact that 
        \[h_2(\lambda) = K_2 \max \left\{ 0,  \frac{1}{\lambda} - \frac{1}{\lambda_{\max}}\right\} \in \left[0 , K_2 \left(1-\frac{1}{\lambda_{\max}} \right) \right],\]
    when \(\lambda \in \left[ 1, \infty \right)\).
\end{proof}

The drift of \(Z\) is known from Lemma~\ref{lem:drift_Z}, so to compute the drift of \(G\) it suffices to compute that of \(H_2\). As before, we abbreviate $x=x^t$ and $\lambda=\lambda^t$ where the index is clear from the context.
\begin{sclaim}
    \label{clm:conditional_drift_3_c1}
    At all times \(t \ge 0\) with $Z^t >0$ we have
    \begin{align*}
        \EE \left[ H_2^t-H_2^{t+1} \mid x, \lambda \right] \ge - \frac{K_2}{\lambda} (F-1) \Pr \left[B\right] + \frac{K_2}{\lambda}\left(1-F^{-1/s}\right) \Pr \left[ \bar{B} \right] \, \indicator{\lambda < n}. 
    \end{align*}
\end{sclaim}

\begin{proof}[Proof of Claim~\ref{clm:conditional_drift_3_c1}]
    Similar to the proof of Claim~\ref{clm:small_sc_drift_h}, we analyse the drift of~\(H_2^t\). We first give a general bound that we will use for the case that the fitness increases. We have $\lambda^{t+1} \ge \lambda^t/F$ and thus $ (\lambda^{t})^{-1}- (\lambda^{t+1})^{-1} \ge (1-F)/\lambda^t$. Hence, $H_2^t-H_2^{t+1} \ge K_2(1-F)/\lambda^t$ for all $t$. In particular,
    \begin{align*}
        \EE [ H_2^t - H_2^{t+1} \mid x^t, \lambda^t, f(x^{t+1}) > f(x^t) ]
            &\ge K_2 \frac{(1 - F)}{\lambda^t}.
    \end{align*}
    For $f(x^{t+1}) \le f(x^t)$, we have $\lambda^{t+1} \ge \lambda^t$, and thus $H_2^{t}-H_2^{t+1} \geq 0$. If additionally $\lambda^t < \lambda_{\max} F^{-1/s} = n$, then $\lambda^t < F^{1/s}\lambda^{t} = \lambda^{t+1} \le \lambda_{\max}$, and hence $H_2^{t}-H_2^{t+1} = K_2(1-F^{-1/s})/\lambda^t$. Summarizing,
    \begin{align*}
        \EE [ H_2^t - H_2^{t+1} \mid x^t, \lambda^t, f(x^{t+1}) \le f(x^t) ]
            &\ge\begin{cases}
                    K_2 (1 - F^{-1/s})/\lambda^t& \text{if \(\lambda_t < n\);}\\
                    0 & \text{if \(n \le \lambda_t\).}\\
                \end{cases}
    \end{align*}
    This gives 
    \begin{align*}
        \EE[ H_2^t - H_2^{t+1} ] & \ge  K_2  (1 - F)/\lambda^t\cdot \Pr[f(x^{t+1}) > f(x^t)] \\
        & \quad + K_2 (1-F^{-1/s})/\lambda^t \cdot \Pr[f(x^{t+1}) \le f(x^t)] \cdot \indicator{\lambda^t < n},
    \end{align*}
    where the conditioning on \(x, \lambda\) is implicit.
    Recall that \(B\) is the event that at least one of the offspring flips a 0-bit of $X_t$, which is a necessary condition for \(f(x^{t+1}) > f(x^t)\). Also, we have $1-F<0$ and $1-F^{-1/s}>0$ due to $F>1$, so replacing $\Pr[ f(x^{t+1}) > f(x^t) ]$ by its upper bound $\Pr[B]$ and replacing $\Pr[ f(x^{t+1}) \le f(x^t) ]$ by its lower bound $\Pr[\bar{B}]$, we conclude the proof.
\end{proof}

We can bound the drift of $G_2^t$ from below as follows.
\begin{corollary}
    \label{cor:drift_G2}
    There exists constants \(0 < s_0\) and \(1 < F_0\) such that the following holds. For all \(0< s \le s_0\) and all \(1 < F \le F_0\) there exists a choice of \(K_2\) and a constant \(\delta > 0\) such that for all times $t$ with $Z^t >0$,
    \begin{align*}
        \EE [ G^t_2 - G^{t+1}_2 \mid x, \lambda ] \ge \delta G_2^t/n.
    \end{align*}
\end{corollary}

\begin{proof}
    We will first show 
    \begin{align}\label{eq:cor:drift_G2}
        \EE [ G^t_2 - G^{t+1}_2 \mid x, \lambda ] \ge \delta' Z^t/n
    \end{align}
    for some $\delta'>0$. Combining Lemma~\ref{lem:drift_Z} together with Claim~\ref{clm:conditional_drift_3_c1} we obtain that the drift of \(G_2\) is at least
    \begin{align*}
        \EE [ G_2^t - G_2^{t+1} \mid x, \lambda ] 
        &  \ge  \Pr[B] \Big( \alpha_1 \frac{Z^t}{n} - K_2(F-1) \frac{1}{\lambda} \Big) - \alpha_2 e^{-\beta \lambda} \\
            \numberthis        
            & \qquad \qquad \qquad + \Pr[\bar{B}] K_2 (1-F^{-1/s}) \cdot \frac{1}{\lambda} \indicator{\lambda \le n}, \label{eq:cor:drift_G2:full_drift}
    \end{align*}
    for some constants \(\alpha_1, \alpha_2, \beta>0\).
    
    We will argue that if \(F>1\) and \(s>0\) are both small enough and if \(K_2\) is chosen appropriately, then the drift of \(G_2^t\) is of order \(Z^t / n\). We will choose~\(F, s\) later but we may already choose \(K_2 = \alpha_1 / (2(F-1)) \) so that the drift is at least 
    \begin{align*}
        \EE [ G_2^t - G_2^{t+1} \mid x, \lambda ]
            &\ge \Pr[B] \alpha_1 \left( \frac{Z^t}{n} - \frac{1}{2\lambda} \right) - \alpha_2 e^{-\beta \lambda} \\
            & \qquad \qquad \qquad + \Pr[\bar{B}] \alpha_1 \frac{1-F^{-1/s}}{2(F-1)} \cdot \frac{1}{\lambda} \indicator{\lambda \le n}.
    \end{align*}
    
    Our proof is based on a case distinction. Let \(b_3\) be the constant of Lemma~\ref{lem:properties_AB}, \(\gamma := 1 - e^{-b_3}\) and let
    \(\tilde{\lambda}>0\) be such that \(\gamma \alpha_1/(4 {\lambda}) - \alpha_2 e^{- \beta {\lambda}} \geq 0\) holds for all $\lambda \ge \tilde \lambda$. 
    
    \case{If \(\lambda \le \max\{ \tilde{\lambda}, n/Z^t \}\)} then by ignoring the first positive contribution in \eqref{eq:cor:drift_G2:full_drift}, the drift is at least 
    \begin{align*}
        \EE [ G_2^t - G_2^{t+1} \mid x, \lambda ]
            &\ge - \Pr[B] \alpha_1 \frac{1}{2\lambda} - \alpha_2 e^{-\beta \lambda} + \Pr[\bar{B}] \alpha_1 \frac{1-F^{-1/s}}{2(F-1)} \cdot \frac{1}{\lambda}.
    \end{align*}
    Splitting the positive contribution into three equal parts gives 
    \begin{align}\label{eq:c_one_small_lambda}
    \begin{split}
        \EE [ G_2^t - G_2^{t+1} \mid x, \lambda ]
            &\ge \Pr[\bar{B}] \frac{1}{3} \alpha_1 \frac{1-F^{-1/s}}{2(F-1)} \cdot \frac{1}{\lambda} \\
                &\qquad\qquad + \left( \Pr[\bar{B}] \frac{1}{3}\cdot \frac{1-F^{-1/s}}{(F-1)} - \Pr[B]  \right)  \frac{\alpha_1}{2\lambda} \\
                &\qquad\qquad + \Pr[\bar{B}] \frac{1}{3} \alpha_1 \frac{1-F^{-1/s}}{2(F-1)} \cdot \frac{1}{\lambda} - \alpha_2 e^{-\beta \lambda}.
    \end{split}
    \end{align}
    Recall that Lemma~\ref{lem:properties_AB} guarantees that \(\Pr[\bar{B}] = e^{-\Theta(\lambda Z^t/ n)}\). 
    For the currently considered range of \(\lambda\), we have \(\Pr[\bar{B}] = \Omega(1)\), while \(\Pr[B] = O(1)\). Also, \((1 - F^{-1/s})/(F-1) \overset{s \to 0}{\rightarrow} 1/(F-1) \overset{F \to 1}{\rightarrow} + \infty\), so a choice of \(F,s\) small enough (but constant) guarantees that the second and third line in~\eqref{eq:c_one_small_lambda} are both non-negative. This means that in the range of \(\lambda\) considered, the drift is at least some multiple of \(1/\lambda\), which is at least \(\delta' Z^t / n\) for a small enough constant~\(\delta'\).
    
    \case{If \(\lambda > \max\{\tilde{\lambda}, n/Z^t\}\)} then by ignoring the last positive contribution in \eqref{eq:cor:drift_G2:full_drift} we see that the drift is at least 
    \begin{align*}
        \EE [ G_2^t - G_2^{t+1} \mid x, \lambda ]
            &\ge \Pr[B] \alpha_1 \left( \frac{Z^t}{n} - \frac{1}{2\lambda} \right) - \alpha_2 e^{-\beta \lambda} \\
            &\ge \Pr[B] \alpha_1 \frac{Z^t}{2n} - \alpha_2 e^{-\beta \lambda} \\
            &\ge \Pr[B] \alpha_1 \frac{Z^t}{4n} + \Pr[B] \alpha_1 \frac{1}{4\lambda} - \alpha_2 e^{-\beta \lambda} .
    \end{align*}
    Lemma~\ref{lem:properties_AB} states that \(\Pr[B] \ge 1 - e^{-b_3 \lambda Z^t / n} \ge 1 - e^{-b_3} = \gamma\) since \(\lambda > Z^t/n\) and by definition of \(\gamma\).
    Since \(\lambda > \tilde{\lambda}\), the last two contributions sum up to a non-negative constant so that the drift is at least \(\gamma \alpha_1 \frac{Z^t}{4n} \ge \delta Z^t/n\).
    
    At any time when $Z^t \ge 1$ and $\lambda^t \ge 1$ we have
    \begin{equation*}
        G_2^t = Z^t + K_2 \max\{0, 1/\lambda - 1/{\lambda}_{\max}\} \le Z^t (1 + K_2),
    \end{equation*}
    so $Z^t \ge G_2^t / (1+K_2)$. Letting \(\delta' = \delta / (K_2 + 1)\) we have for all $t \ge 0$ with~$Z^t >0$
    \begin{align*}
        \EE [G_2^t - G_2^{t+1} \mid x, \lambda] \ge \delta' G_2^t.
    \end{align*}
    This proves~\eqref{eq:cor:drift_G2}. To relate $Z^t$ to $G_2^t$, recall that by Lemma~\ref{lem:sandwich_c1}, we have $Z^t \ge G_2^t - K_2$. To obtain a multiplicative drift, we distinguish two cases. If $Z^t \ge K_2$, then we have $Z^t \ge Z^t/2 +K_2/2 \ge (G_2^t-K_2)/2 + K_2/2 = G_2^t/2$. On the other hand, if $0< Z^t < K_2$, then we have $Z^t \ge 1 = 2K_2/(2K_2) \ge G_2^t/(2K_2)$. Hence, for all times $t<T$ we have 
    \begin{align*}
        \EE \left[ G^t_2 - G^{t+1}_2 \mid x, \lambda \right] \ge \delta' Z^t/n \ge \frac{\delta' G_2^t}{n\max\{2,2K_2\}},
    \end{align*}
    which concludes the proof.
\end{proof}

Now we are ready to prove the main result of this subsection.
\begin{proof}[Proof of Theorem~\ref{thm:generations_c1}]
    Corollary~\ref{cor:drift_G2} guarantees that at all times before the time $T$ when the optimum is found,
    \begin{align*}
        \EE \left[ G^t_2 - G^{t+1}_2 \mid x, \lambda \right] \ge \delta G_2^t/n,
    \end{align*}
    for a constant \(\delta > 0\). Moreover, $G_2^t \le Z^t \le n+K_2$ holds for all $t\ge 0$ by Lemma~\ref{lem:sandwich_c1}. Let $g_{\text{min}} := \inf \{g_2(x, \lambda) \mid x\in \{0,1\}^n, \ZM(x) >0 , \lambda  \geq 1\} = K_2$. By the Multiplicative Drift Theorem~\ref{thm:MultiplicativeDrift},
    \begin{align*}
        \EE [T] &\le  \frac{n \left(1+\log \big(\frac{g_2(x^0, \lambda^0)}{g_{\min}}\big)\right)}{\delta} 
        \le \frac{n\left( 1+ \log\big(\frac{n+K_2}{K_2}\big)\right)}{\delta} = O(n\log n),
    \end{align*}
    which concludes the proof.
\end{proof}


\subsection{Bound on the Number of Evaluations}\label{sec:eval_small_s}

We have proved that the number of generations is respectively \(O(n)\) or~\(O(n \log n)\) if \(c < 1\) or \(c = 1\). We will now turn our attention to the total number of function evaluations. For $c=1$, we again need to assume that~$F$ is sufficiently close to one. More precisely, we will show the following theorems.

\begin{theorem}
    \label{thm:num_evaluations_csmall}
    Let \(0<c < 1 < F\) be constants. Then there exist constants \(C, s_0 > 0\) such that for all~\(s \le s_0\) and every dynamic monotone function, the expected number of function evaluations of the \saolea with success rate $s$, update strength $F$ and mutation probability \(c/n\) is at most $C n \log n$.
\end{theorem}

\begin{theorem}
    \label{thm:num_evaluations_cone}
    There exist constants \(C, s_0 >0\) and $F_0>1$ such that for all \(s \le s_0\), all $1 < F <F_0$ and every dynamic monotone function, the expected number of function evaluations of the \saolea with success rate $s$, update strength $F$ and mutation probability \(1/n\) is at most $C n^2 \log \log n$.
\end{theorem}
\begin{remark}
    Theorem~\ref{thm:num_evaluations_csmall} is tight since any unary unbiased algorithm needs at least $\Omega(n\log n)$ function evaluations to optimize \onemax~\cite{lehre2012black}. On the other hand, Theorem~\ref{thm:num_evaluations_cone} is not tight. Calculating a bit more precisely would allow to replace the $\log \log n$ factor by an even smaller factor. However, we suspect that even the main order $n^2$ is not tight, since the \ooea with~$c=1$ is known to need time~$O(n^{3/2})$ even in the pessimistic PO-EA model~\cite{colin2014monotonic}, which includes every dynamic monotone function. The order~$n^{3/2}$ is tight for the PO-EA, but a stronger bound of $O(n\log^2 n)$ is known for all static monotone functions~\cite{lengler2019does}, and an $O(n\log n)$ bound is known for all dynamic linear functions~\cite{lengler2018noisy}. 
    
    We conjecture that the number of function evaluations required to opti\-mise \emph{static} monotone functions is linear up to some logarithmic factors (in fact, we conjecture $O(n \log n)$) even for $c=1$. However, the methods used in~\cite{lengler2019does} are rather different from the ones in this paper, so it remains unclear whether they can be transferred.

    We also conjecture that \emph{dynamic} monotone functions are harder to opti\-mise, i.e., that \(O(n \log n)\) generations and \(O(n^{3/2})\) evaluations are tight. More precisely, we conjecture that the `adversarial' \dynBV des\-cribed in our conclusion is the hardest dynamic monotone function for the \saolea, and requires \(\Omega(n \log n)\) generations and \(\Omega(n^{3/2})\) evaluations. 
\end{remark}

\begin{remark}\label{remark:error}
Our approach uses the best-so-far \zeromax value $Z^t_*$ defined below as in~\cite{hevia2021arxiv,hevia2021self}. However, apart from that our proof is rather different. In fact, we believe that the proof in these papers is not fully correct. In the proof of Theorem~3.5 in~\cite{hevia2021arxiv}, the authors bound the number of evaluations per generation by identically distributed random variables, and use Wald's equation to bound the total number of evaluations. However, Wald's equation is only true for the sum of \emph{independent} random variables (or for similar conditions, e.g.~\cite{doerr2015optimizing}), a condition that is not satisfied in this situation. (The random variables are identically distributed, but not \emph{independently} iden\-ti\-cally distributed.) 
Thus we need to use a different approach.
\end{remark}

To avoid the issue mentioned above, we will decompose the interval \([n]\) into smaller `sub-intervals' and we will show that with very high probability, the time needed for \(\ZM(x)\) to `traverse' such an interval is of the expected order. We will compute the expected number of children at each of those steps, and will conclude using linearity of expectation.

To prove concentration of the time needed to traverse `sub-intervals' we will use Theorem~\ref{thm:subgaussian_hit_fast} in the case \(c < 1\) and Theorem~\ref{thm:multiplicative_drift_concentration} in the case \(c = 1\). The key ideas are summarised in the following lemmas. The first one is an adaptation of one proved by Hevia Fajardo and Sudholt. Below, we let \(Z^t_* := \min_{t' \le t} (Z^{t'})\) be the smallest value of \(Z^t\) observed until time $t$. Naturally, the process is unaware of this value but it will turn out useful for the analysis. We will apply the following lemmas to intermediate stages of a run, so we will consider an arbitrary starting population size $\linit$ in them. 

\begin{lemma}[Fajardo, Sudholt~\cite{hevia2021arxiv,hevia2021self}]
    \label{lem:expectation_lambda}
    Consider the \saolea as in Theorems~\ref{thm:num_evaluations_csmall} or~\ref{thm:num_evaluations_cone}, with an arbitrary initial search point and an initial value of \(\lambda = \linit\). There exists a constant \(C > 0\) such that at all times $t\ge 0$ and for all \(z > 0\) we have 
    \begin{align*}
        \EE \left[ \lambda^t \cdot \indicator{Z^t_* \ge z} \right] \le \linit / F^t + Cn/z.
    \end{align*}
\end{lemma}
    
\begin{lemma}
    \label{lem:fast_return_to_normal}
    Consider the self-adjusting \((1, \lambda)\)-EA as in Theorems~\ref{thm:num_evaluations_csmall} or~\ref{thm:num_evaluations_cone}, with an arbitrary initial search point and an initial value of \(\lambda = \linit\). Let \(T\) denote the first time \(t\) at which \(\lambda^t \le 8e n \log n /Z^t\). There exists an absolute constant \(C > 0\) such that 
    \begin{align*}
        \EE \left[ \sum_{t = 1}^{T}{\lambda^t} \right] \le C \linit.
    \end{align*}
\end{lemma}

\begin{lemma}
    \label{lem:cross_fast_whp}
    Let \((a,b)\) be an interval of length \(b-a = \log n\). 
    Consider the self-adjusting \((1, \lambda)\)-EA with $c<1$ as in Theorem~\ref{thm:num_evaluations_csmall}, with an initial search point \(x = \xinit\) such that \(\ZM(\xinit) \le b\), and an arbitrary initial value of \(\lambda\).
    
    Let \(T\) be first time \(t\) at which \(Z^t \le a\). Then there exists an absolute constant \(D > 0\) such that \(T \le D \log n\) with probability at least \(1 - n^{-4}\).
\end{lemma}

\begin{proof}[Proof of Lemma~\ref{lem:expectation_lambda}]
    We will compute the expectation using the following formula
    \begin{align}\label{eq:lem:expectation_lambda2}
        \EE [ \lambda^t \cdot \indicator{Z_*^t \ge z} ] \le 1+\sum_{ \ell=1 }^{\infty }{ \Pr\left[ \lambda^t \ge \ell, Z^t_* \ge z \right] }.
    \end{align}
    
    Let \(\ell\) be an integer; if \(\ell \le \max\{ \linit / F^t, n/z\}\), then \(\Pr[ \lambda^t \ge \ell ] = 1\). Otherwise, observe that for \(\lambda^t \ge \ell\) there has to exist a time before \(t\) when \(\lambda\) increases, i.e.\ when the fitness does not improve. We may then write 
    \begin{align}
        \indicator{\lambda^t \ge \ell, \; Z^t_* \ge z}
            &= \sum_{k = 1}^{t}{ \indicator{\lambda^t \ge \ell, \; Z^t_* \ge z, \; \text{\(t-k\) is the last time when \(\lambda\) increases} } }.  \label{eq:lem:expectation_lambda}
    \end{align}
    Note that if \((t-k)\) is the last time when \(\lambda\) increases, then the number of children at this time must be \(\lfloor\lambda^t \cdot F^{k - 1/s}\rceil\ge -1 + \ell \cdot F^{k - 1/s} \). Naturally, if \(Z^t_* \ge z\), we must also have \(Z^{t-k}_* \ge z\). In particular, we find that if the following event holds
    \begin{align*}
        \left\{ \lambda_t \ge \ell, \quad Z^t_* \ge z, \quad \text{\(t-k\) is the last time when \(\lambda\) increases} \right\},
    \end{align*}
    then so does 
    \begin{align}\label{eq:event_evaluations2}
        \{\lambda_{t-k} \ge  \ell  F^{k - 1/s}, \ Z^{t-k}_* \ge z, \ \text{no improvement made at time \(t-k\)}\}.
    \end{align}
    If \(Z_*^{t-k} \ge z\), the probability of a single child improving the fitness is at least 
    \begin{align*}
        \left( 1-c/n \right)^{n-Z^{t-k}} \cdot \left( 1 - \left( 1 - c/n \right)^{Z^{t-k}} \right) \ge e^{-c} \cdot \left( 1 - e^{-cz/n} \right) \ge \frac{cz}{2e^{c}n},
    \end{align*}
    where the last step holds by Lemma~\ref{lem:basic} since $cz/n \le 1$. In particular, this implies that the probability of the event in~\eqref{eq:event_evaluations2} is at most
    \begin{align*}
        &\Pr \left[ \lambda_{t-k} \ge \ell F^{k - 1/s}, \quad Z_*^{t-k} \ge z, \quad \text{no improvement is made at time \(t-k\)} \right] \\
            &\qquad \qquad \qquad \le \left(1 - \frac{cz}{2e^{c}n} \right)^{\lfloor\ell F^{k - 1/s}\rfloor}. 
    \end{align*}
    Replacing in \eqref{eq:lem:expectation_lambda} and using $F^k = e^{k\log F} \ge 1+k\log F$ gives
    \begin{align*}
        \Pr \left[ \lambda_t = \ell, \ZM^*_t \ge z \right] 
            &\le \sum_{ k=1 }^{t}{ \left( 1 - \frac{cz}{2e^{c}n} \right)^{\lfloor\ell \cdot F^{k - 1/s}\rfloor}} \\
            &\le \sum_{k=1}^\infty{ \left( 1 - \frac{cz}{2e^{c}n} \right)^{-1+\ell \cdot (1 + k \log F) F^{-1/s}} } \\
            &= \left( 1 - \frac{cz}{2e^{c}n}\right)^{-1+\ell  F^{-1/s}} \cdot \frac{\left( 1 - \frac{cz}{2e^{c}n} \right)^{\ell  \log F F^{-1/s}}}{1 - \left( 1 - \frac{cz}{2e^{c}n} \right)^{\ell  \log F F^{-1/s}}} \\
            &\le C \left( 1 - \frac{cz}{2e^{c}n}\right)^{\ell F^{-1/s}},
    \end{align*}
    for a sufficiently large constant \(C\) since \(\ell \ge n/z\). Now using Equation~\eqref{eq:lem:expectation_lambda2} and taking the trivial upper bound of \(1\) for the probability of the first \(\max\{ n/z, \linit / F^t\}\) terms, we obtain
    \begin{align*}
        \EE \left[ \lambda^t \cdot \indicator{Z^t_* \ge z} \right] 
            &\le 1+ \frac{\linit}{F^t} + n/z + C \sum_{\ell=n/z}^{\infty}{ \left(1 - \frac{cz}{2e^{c}n} \right)^{\ell F^{-1/s}} }\\
            &\le 1+ \frac{\linit}{F^t} + n/z + C \sum_{\ell=0}^{\infty}{ \left(1 - \frac{cz}{2e^{c}n} \right)^{\ell F^{-1/s}} }\\
            &\le 1+ \frac{\linit}{F^t} + n/z + C \frac{1}{1 - \left( 1 - \frac{cz}{2e^{c}n} \right)^{F^{-1/s}}} \\
            &\le 1+ \frac{\linit}{F^t} + n/z + C\frac{2 e^{c} F^{1/s} n}{cz}\left(1+\frac{cz}{2e^cnF^{1/s}}\right) \\
            &\le \frac{\linit}{F^t} + C'\frac{n}{z},
    \end{align*}
    for a large constant \(C'>0\).
\end{proof}

\begin{proof}[Proof of Lemma~\ref{lem:fast_return_to_normal}]
    Consider a time \(t < T\) such that \(\lambda^t \ge 8e n \log n/Z^t\). The probability that a child improves the fitness is at least 
    \[
        \left(1 - c/n\right)^{n - Z^t} \left(1 - \left( 1 - c/n \right)^{Z^t} \right) \ge e^{-c} \left( 1 - e^{ -cZ^t / n}\right) \ge cZ^t / (2e^{c}n),
    \] 
    where the last step holds by Lemma~\ref{lem:basic} since $cZ^t/n \le 1$. Hence the probability that all the \(\round{\lambda^t}\) children fail to improve the fitness is at most \( \left( 1 - cZ^t/(2e^cn) \right)^{\lfloor{\lambda^t}\rceil} \le \exp \left( - c Z^t \lfloor{\lambda^t}\rceil / (2e^cn) \right) = o(1)\). From this, we easily compute 
    \begin{align*}
        \EE \left[ \lambda^{t+1} \cdot \indicator{{t+1} < T} \mid \lambda_t \right] 
            = (1 - o(1)) \lambda^t / F + o(1) \lambda^t F^{1/s} 
            \le \lambda^t / F^{1/2}.
    \end{align*}
    This recursively implies that \(\EE\left[ \lambda^t \cdot \indicator{t < T} \right] \le \linit \cdot F^{-t/2}\). Using $\lambda^{t} \le \lambda^{t-1}F^{1/s}$, we can now conclude,
    \begin{align*}
        \EE \left[ \sum_{t=1}^{T}{\lambda^t} \right]
            &= \EE\left[ \sum_{t=1}^{\infty}{ \lambda^t \cdot \indicator{t \le T} } \right] \le \EE\left[ F^{1/s}\sum_{t=1}^{\infty}{ \lambda^{t-1} \cdot \indicator{t \le T} } \right] \\
            & = F^{1/s}\sum_{t=0}^{\infty}{ \EE\left[\lambda^t \cdot \indicator{t < T}\right] }
            \le F^{1/s}\sum_{t=0}^{\infty}{ \linit F^{-t/2} } 
            \le C \linit,
    \end{align*}
    for some constant \(C\).
\end{proof}

\begin{proof}[Proof of Lemma~\ref{lem:cross_fast_whp}]
    Recall the definition 
    \[
        g_1(x, \lambda) = \ZM(x) + h_1(\lambda) = \ZM(x) + K_1 \max\{0,\log_F( \lambda_{\max} / \lambda )\},
    \] 
    and $G_1^t := g_1(x^t,\lambda^t)$. We define \(\Gamma^t := \min\{1+K_1/s,G_1^t - G_1^{t+1}\}\). 
    We observe that \(\sum_{t=0}^{ \tau}{ \Gamma^t } \le G_1^0 - G_1^{\tau} \le K_1 \log_F \lambda_{\text{max}} + b - Z^{\tau}\), so if we define \(T'\) as the first time \(\tau\) when \(\sum_{t=1}^{\tau}{ \Gamma^t } \ge K_1 \log_F \lambda_{\text{max}} + b - a\), we clearly have \(Z^{\tau} \le a\). In other words, \(T \le T'\) and we will show the desired tail bound for \(T'\).

    To obtain the tail bound for $T'$, we observe that \(h_1\) is decreasing and that \(h(\lambda) - h(\lambda F) \le K_1\) for all \(\lambda\). By Corollary~\ref{cor:drift_G1}, the drift of \(\Gamma_t\) is at least a constant \(\varepsilon > 0\) and Lemma~\ref{lem:improvements_subgaussian} applied with \(C_1 = K_1 / s\), \(C_2 = K_1\) guarantee that $\varepsilon - \Gamma^t$ is sub-Gaussian. 
    Thus Theorem~\ref{thm:subgaussian_hit_fast} is applicable. Let \(D\) be the constant from Theorem~\ref{thm:subgaussian_hit_fast}. Then we choose \(\tau = \max \{4/D, 2/\varepsilon\} \cdot \log n\) and Theorem~\ref{thm:subgaussian_hit_fast} immediately implies that \( \Pr[T > \tau] \le \Pr[ T' > \tau ] \le n^{-4} \).
\end{proof}

We are now ready to prove Theorem~\ref{thm:num_evaluations_csmall}.
We look at a slight alteration of the \saolea, working in exactly the same way as the `normal' process except that we introduce some \emph{idle steps} in which the algorithm does not do anything. Moreover, we divide a run of the algorithm into \emph{blocks} and \emph{phases} as follows. For simplicity, we will assume in the following that $n/\log n$ is an integer. 
A block starts with an initialisation phase which lasts until the condition \(\lambda^t \le F^{1/s}8e n \log n/Z^t\) is met. Once this phase is over, the block runs for \(n / \log n\) phases of length \(D \log n\), with \(D\) the constant of Lemma~\ref{lem:cross_fast_whp}. During the \(i\)-th such phase the process attempts to improve~\(Z^t\) from \(n - i \log n\) to \(n - (i+1) \log n\). If such an improvement is made before the \(D \log n\) steps are over, then the process remains idle during the remaining steps of that phase. We call the non-idle steps \emph{active}. 

If a phase fails to make the correct improvement in $D\log n$ generations, or if \(\lambda^{t} \ge F^{1/s}8 e n \log n / Z_*^t\) at any point after the initialisation phase is over, then the whole block is considered a failure, and the next block starts. Obviously, the entire process stops (and succeeds) if the optimum is found. With this partitioning of a run, we will prove Theorem~\ref{thm:num_evaluations_csmall}.

\begin{proof}[Proof of Theorem~\ref{thm:num_evaluations_csmall}]
    The proof relies on the following two facts:
    \begin{enumerate}[(i)]
        \item every block finds the optimum whp; \label{itm:num_evaluations_proof_1}
        \item consider a block starting with \(\lambda = \linit\), then the expected number of function evaluations during this block is at most \(K (\linit + n \log n)\) for a constant \(K = K(c,F,s)\). \label{itm:num_evaluations_proof_2}
    \end{enumerate}
    It is rather easy to see how those two items imply the theorem. The algorithm starts with an initial value of \(\lambda = 1\), so the expected number of function evaluations in the first block is at most \(O(n \log n)\). Recall that a block terminates as soon as \(\lambda\) goes above \(F^{1/s}8e n \log n / Z_*^t\) after the initialisation phase. This means that any block-run after the first one will start with \(\linit \le F^{2/s} 8 e n \log n / Z_*^t\) and by \ref{itm:num_evaluations_proof_2} its expected total number of evaluations is also \(O(n \log n)\). The success of each block is at least \(1 - o(1)\) for all possible \(\xinit, \linit\) it starts with, so by \ref{itm:num_evaluations_proof_1} we have an expected \((1 + o(1))\) blocks, each requiring an expected \(O(n \log n)\) evaluations, hence the result. 
    
    To conclude, we will now prove the two items above. Let us start with \ref{itm:num_evaluations_proof_1} which is simpler: the initialisation never fails since it runs until it succeeds, i.e., until $\lambda$ gets small enough or the optimum is found. By Lemma~\ref{lem:cross_fast_whp}, `crossing' any interval of size \(\log n\) fails with probability at most \(n^{-4}\). By union bound over all \(n/\log n\) phases of the block, the probability that one of them fails for this reason is at most \(n^{-3} \log^{-1} n = o(n^{-3})\). Finally, the block might also fail because we have \(\lambda^{t} \ge F^{1/s} 8 e n \log n / Z_*^t\) at some point, which means that the $(t-1)$-th step was not successful despite \(\lambda^{t-1} \ge 8e n \log n / Z^{t-1}\), since $Z_*^t\le Z^{t-1}$. This happens with probability at most \(n^{-2}\) by Lemma~\ref{lem:properties_AB}, so by union bound over the \(D n\) generations in the different phases, the probability that this happens at some time during a given block is at most \(o(1)\).
    
    We now prove \ref{itm:num_evaluations_proof_2}. Consider any \(1 \le i \le n / \log n\); we will show that the number of function evaluations in the \(i\)-th phase of the block is at most of order \(\frac{n}{n - i \log n}\). In a slight abuse of notation, for $t \in [D\log n]$ we will let \(\lambda_i^t\) denote the value of $\lambda$ in the \(t\)-th step \emph{of the $i$-th phase}, and we set \(\lambda_i^t\) to be \(0\) if step \(t\) is idle, e.g.\ if the improvement to \(Z^t \le n - (i+1) \log n\) has already been found, or if the $i$-th phase does not happen because the block has failed before that. 
    
    Since the previous phase is successful, the value of \(\lambda\) at the start of this phase must be \(\lambda_{i}^0 \le C \frac{n \log n}{n - i \log n}\).
    By definition, \(t \) is only active if $Z^t$ has never been at or below \(n - (i+1) \log n\), so by Lemma~\ref{lem:expectation_lambda} we have
    \begin{align*}
        \EE \left[ \lambda^t_i \right]
            &\le C'\left( \lambda_{i}^0/F^{t} + \frac{n}{n - (i+1) \log n} \right).
    \end{align*}
    Note that in each round, the number of function evaluations is $\lfloor \lambda_i^t\rceil \le 1+\lambda_i^t$. Summing over the \(D \log n\) steps of the \(i\)-th phase, we see that the total number of function evaluations during this phase satisfies 
    \begin{align*}
        \EE\left[ \sum\nolimits_{t=1}^{D\log n}{(1+\lambda^t_i)} \right] \le C'' \frac{n \log n}{n - (i+1) \log n},
    \end{align*}
    for a constant \(C''>0\).
    Hence, excluding the initialisation phase, the total number of function evaluations is in expectation at most 
    \begin{align*}
        C''\sum_{i = 1}^{n / \log n}{ \frac{n \log n}{n - i \log n} }
            &= C''\sum_{ j = 1 }^{ n / \log n }{\frac{n \log n}{j \log n}} 
            \le C''' n \log n
    \end{align*}
    for a new constant $C'''>0$. Lemma~\ref{lem:fast_return_to_normal} immediately gives that the expected number of function evaluations of the initialisation phase is at most \(C'\linit\), and \ref{itm:num_evaluations_proof_2} is proved. 

\end{proof}

The proof of Theorem~\ref{thm:num_evaluations_cone} is extremely close to that above, so we will only give a sketch of it. The main difference is the way the phases of a block are defined: during the \(i\)-th phase, the algorithm attempts to improve~$Z^t$ from~\(n / \log^{i-1} n\) to \(n / \log^{i} n \) in \(D n \log \log n\) steps for a large constant \(D\). One checks that $n/\log^{i} n = n/e^{i\log\log n}$ is less than $1$ for $i > \log n/\log \log n$, so we have $1+\log n/\log \log n$ phases in a block. 
Since \(G_2^t \) has multiplicative drift, Theorem~\ref{thm:multiplicative_drift_concentration} immediately gives that the probability that a phase fails to improve \(Z^t\) is at most \(\log^{-2} n\), so the probability that any phase in a fixed block fails is \(O( \log n / \log \log n \cdot \log^{-2} n ) = o(1)\). In the \(i\)-th phase, the expected value of \(\lambda\) at each step is at most \(O( \log^i n )\). The total number of function evaluations per block after the initialization phase is then
\begin{align*}
    O(1)\sum_{i = 0}^{ \log n / \log \log n }{ n \log \log n \cdot \log^i n } =O( n^2 \log \log n),
\end{align*}
which implies Theorem~\ref{thm:num_evaluations_cone}.

\section{Close to the Optimum Success Rates Become Asymptotically Irrelevant}
\label{sec:efficient_near_optimum}

In this section, we will show that one can still have efficient search even when \(s\) is large, provided one starts close enough to the optimum. More precisely, we will prove the following theorem. For simplicity, we only treat the case $c<1$.

\begin{theorem}
    \label{thm:large_s_efficient}
    Let \(0 < c < 1 < F\) be constants. For every \(s > 0\), there exists an \(\varepsilon > 0\) such that for any initial search point $x^0$ satisfying \(\ZM(x^0) / n \le \varepsilon\) and for any initial population size \(\linit \ge 1\) the following holds. For every dynamic monotone function, with high probability the number of generations of the \saolea with success rate $s$, update strength $F$, and mutation probability \(c/n\) and initial state $(x^0,\linit)$ is \(O(n)\). Additionally, the number of function evaluations is \(\omega(n \log n)\) with probability \(o(1)\).

\end{theorem}

The approach will be essentially the same as in Section~\ref{subsec:generations_csmall}. The main difference lies in the potential function: we will need to introduce a second penalty term into the part \(h(\lambda)\) that depends on $\lambda$. Moreover, when \(s\) is large, no potential function can have strong positive drift towards the optimum for all values of \(Z^t\), as it would otherwise contradict the negative results from Section~\ref{sec:inefficient}.
Hence, we will only show that the potential has positive drift when \(Z^t/n \le 2\varepsilon\). Then by the Negative Drift Theorem, starting from $Z^t/n \le \eps$ it is unlikely that the exploration reaches a search point for which \(Z^t / n > 2\varepsilon\) in polynomial time. Hence, the algorithm stays in a range where the drift is positive, and by the Additive Drift Theorem the optimum is found efficiently.

\begin{definition}[Potential function for positive result near optimum]\label{def:h_positive3}
    For all $\lambda \ge 1$, we set 
    \begin{align*}
    h_3(\lambda) = K_1 \max \left\{0, \log_F \lambda_{\max} / \lambda \right\} + K_2 e^{- K_3 \lambda},
\end{align*}
where $\lambda_{\max} = F^{1/s}n$ and the constants $K_1, K_2, K_3 >0$ will be fixed later. Then for $x\in\{0,1\}^n$ and $\lambda \in [1,\infty)$ we define
    \begin{align*}
        g_3(x, \lambda) := \ZM(x) + h_3(\lambda),
    \end{align*}
    and we set $H_3^t := h_3(x^t,\lambda^t)$ and $G_3^t := g_3(x^t,\lambda^t)$, and as usual $Z^t := \ZM(x^t)$. 
\end{definition}

As before, we get a sandwich lemma.
\begin{lemma}
[``Sandwich'' inequalities]
    \label{lem:sandwich_near_opt}
    For all \(x, \lambda\) we have 
    \begin{align*}
        g_3(x, \lambda) - K_1 \log_F \lambda_{\max} -K_2 \le \ZM(x) \le g_3(x, \lambda).
    \end{align*}
\end{lemma}
\begin{proof}
The proof runs as for Lemma~\ref{lem:sandwich}, except that we also use $K_2 e^{- K_3 \lambda} \in [0,K_2]$.
\end{proof}

As in Section~\ref{sec:efficient}, \(A = A^t\) denotes the event that some child of $x^t$ flips no one-bit and \(B=B^t\) the event that some child of $x^t$ flips (at least) one zero-bit. Also recall that we abbreviate $x = x^t$ and $\lambda=\lambda^t$ when $t$ is clear from the context.
\begin{sclaim}
    \label{clm:large_s_drift_h}
    At all times \(t \ge 1\) with $Z^t >0$ we have 
    \begin{align*}
        \EE [ H_3^t  - H_3^{t+1} \mid x,\lambda ]
            & \ge \indicator{\lambda <n} \Pr[\bar{B}]  \big(K_1/s + K_2 (1 - e^{-K_3 (F^{1/s} - 1)}) e^{-K_3 \lambda} \big) \\
            &\qquad  - \Pr[B] \cdot \big( K_1 + K_2 e^{- K_3 \lambda/F} \big).
    \end{align*}
\end{sclaim}

\begin{proof}
    The proof is extremely similar to that of Claim~\ref{clm:small_sc_drift_h}. In particular, the contribution of the first term \(K_1 \max \left\{ 0, \log_F \lambda_{\max} / \lambda \right\} \) is exactly the same, so we will only compute the contribution of the second term. Since that term is non-negative at time $t$ and is at most $K_2e^{-K_3 \lambda/F}$ at time $t+1$ due to $\lambda^{t+1} \ge \lambda/F$, it contributes at least $-K_2e^{-K_3 \lambda/F}$ to the difference~$H_3^t-H_3^{t+1}$.
 
    In the case \(f(x^{t+1}) \le f(x^t)\) of non-success we have $\lambda^{t+1} = \lambda F^{1/s}$, so we may use the exact contribution to $H_3^t-H_3^{t+1}$, which is 
    \begin{align*}
        K_2 ( 1 - e^{-K_3\lambda  (F^{1/s} - 1)}) e^{-K_3 \lambda} \ge K_2 ( 1 - e^{-K_3  (F^{1/s} - 1)}) e^{-K_3 \lambda},
    \end{align*}
    since $\lambda \ge 1$.
    As in Section~\ref{subsec:efficient_gen}, observing that \(\Pr[ f(x^{t+1}) > f(x^t) ] \le \Pr[B]\) and using the law of total expectation gives the result.
\end{proof}

\begin{corollary}
    \label{cor:large_s_drift_g}
    There exist constants \(\varepsilon, \delta, K_1, K_2, K_3 >0 \) (which may only depend on \(c, F, s\)) such that for all times \(t\) when \(0<Z^t / n \le 2\varepsilon\) we have
    \begin{align*}
        \EE \left[ G_3^t-G_3^{t+1} \mid x, \lambda:  \ZM(x)/n \le 2\varepsilon \right] \ge \delta.
    \end{align*}
    This also holds if we replace \(G_3^t-G_3^{t+1}\) by \(\min \{1 + K_1/s + K_2 \; , \; G_3^t-G_3^{t+1} \}\).
\end{corollary}

\begin{proof}
    Combining Claims~\ref{clm:conditional_drift_1},~\ref{clm:conditional_drift_2}~and~\ref{clm:large_s_drift_h}, there exist constants \(\alpha_1, \alpha_2, \alpha_3, \beta>0 \) (which may only depend on \(c,F,s\)) such that at all times \(t\) with $Z^t >0$ we have 
    \begin{align*}
        \EE [ G_3^t-G_3^{t+1} \mid x, \lambda ]
            & \ge \Pr[B] \left( \alpha_1 - K_1 - \alpha_2 K_2 e^{-K_3 \lambda/F} \right) - \alpha_3 e^{-\beta \lambda } \\
            & \quad + \Pr[\bar{B}] \indicator{\lambda <n} \left( \alpha_4 K_2 K_3 e^{- K_3 \lambda} + K_1/s \right). \numberthis \label{eq:large_s_drift_g_intermediate}
    \end{align*}
    We choose \(K_1 = \alpha_1/2, K_3 = \beta/2, K_2 \ge 2 \alpha_3/(\alpha_4K_3)\).
    We will prove a constant lower bound in two different cases, depending on whether \(\lambda\) is small or not. 

    \case{If \(\lambda \le \frac{F}{K_3} \log \frac{2 \alpha_2 K_2}{K_1}\)} 
    then, ignoring the (positive) contribution of \(\Pr[B](\alpha_1 - K_1)\) and \(\Pr[\bar{B}] K_1/s\), we see that the drift of \(g\) is at least
    \begin{align}\label{eq:cor:large_s_drift_g1}
        \EE [ G_3^t-G_3^{t+1} \mid x, \lambda ] 
            & \ge- \Pr[B] \alpha_2 K_2 e^{-K_3\lambda/F} - \alpha_3 e^{- \beta \lambda}  + \Pr[\bar{B}] \alpha_4 K_2 K_3 e^{-K_3 \lambda} \nonumber\\
            &  \ge - \Pr[B]\alpha_2 K_2
            + \alpha_3 e^{-K_3 \lambda}( 2 \Pr[\bar{B}] - 1), \nonumber \\
            & \ge - \Pr[B]\alpha_2 K_2 + \alpha_3 \left(\frac{K_1}{2\alpha_2 K_2}\right)^F (2\Pr[\bar B] - 1)
    \end{align}
    where in the second step we bounded $e^{-K_3\lambda/F} \le 1$ and used \(\beta \ge K_3\) and \(K_2 \ge 2 \alpha_3 / (\alpha_4 K_3) \), and in the third step we used $\lambda \le F/K_3\cdot\log(2\alpha_2K_2/K_1)$. 
    Recall that \(\Pr[\bar{B}] \ge e^{-b_2 \lambda Z^t / n}\) for some \(b_2 > 0\) by Lemma~\ref{lem:properties_AB}. 
    Since $\lambda$ is bounded by a constant in the current case, if we choose \(\varepsilon\) small enough (w.r.t.\ all previous constants) then we can achieve \(2\Pr[\bar B] - 1 \ge 1/2\) and \(\Pr[B] \alpha_2 K_2 \le \frac{\alpha_3}{4} \left( \frac{K_1}{2 \alpha_2 K_2} \right)^F\). Then the drift must be at least
    \begin{align*}
        \EE [ G_3^t-G_3^{t+1} \mid x, \lambda ] \ge \frac{\alpha_3K_1^F}{4(2\alpha_2K_2)^F}
            \ge \delta,
    \end{align*}
    for a \(\delta >0\) chosen small enough.

\case{If \(\frac{F}{K_3} \log \frac{2 \alpha_2 K_2}{K_1} \le \lambda <n\)}
then the first contribution in \eqref{eq:large_s_drift_g_intermediate} is at least \(\Pr[B] K_1 / 2\). Hence, the drift is at least 
\begin{align*}
    \EE [ G_3^t-G_3^{t+1}\mid x, \lambda ]
        &\ge \Pr[B] \frac{K_1}{2} - \alpha_3 e^{-\beta \lambda } + \Pr[\bar{B}] \left( \alpha_4 K_2 K_3 e^{- K_3 \lambda} + \frac{K_1}{s} \right) \\
        &= \Pr[B] \frac{K_1}{2} + \Pr[\bar{B}] \frac{K_1}{s} + \left( 2\Pr[\bar{B}]  - e^{-K_3 \lambda}  \right) \cdot \alpha_3 e^{-K_3 \lambda} \\
        &\ge \frac{K_1}{\max \{2, s\}} + \left( 2\Pr[\bar{B}]  - e^{-K_3 \lambda}  \right) \cdot \alpha_3 e^{-K_3 \lambda},
\end{align*}
where the second line holds since \(K_3 = \beta / 2\) and \(K_2 = 2\alpha_3 / (\alpha_4 K_3) \). Recall that \(\Pr[\bar{B}] = (1 - c/n)^{\round{\lambda} Z^t} \ge e^{ - 4 c\lambda \varepsilon }\) whenever \(Z^t / n \le 2\varepsilon\) and $n$ is sufficiently large. Choosing \(\varepsilon\) small guarantees that this is larger than \(e^{-K_3 \lambda}\) for all \(\lambda\), meaning that the drift of \(G_3^t\) is at least \(K_1 /\max\{2,s\} \ge \delta\) for a suitable \(\delta>0\).

\case{If \(\lambda \ge n\)} then in \eqref{eq:large_s_drift_g_intermediate} every negative contribution is of order \(e^{-\Omega(n)}\) while the term \(\Pr[B] \left( \alpha_1 - K_1 \right) = \Pr[B] K_1\) is $\Theta(1)$, so the drift is at least \(\delta\) for some \(\delta>0\).


For the last part of the statement, we use the same argument as for Corollary~\ref{cor:drift_G1}, using that \eqref{eq:large_s_drift_g_intermediate} also holds for \(\min\{1, G^t - G^{t+1}\}\) since Claims~\ref{clm:conditional_drift_1} and~\ref{clm:conditional_drift_2} do, and \(h\) may only increase by \(K_1/s + K_2\) at each step.

\end{proof}

We may now prove Theorem~\ref{thm:large_s_efficient}.

\begin{proof}[Proof of Theorem~\ref{thm:large_s_efficient}]
    Let \(\varepsilon, \delta, K_1, K_2, K_3>0\) be the constants from Cor\-ollary~\ref{cor:large_s_drift_g} and assume the initial search point \(x^0\) is such that \(\ZM(x^0)/n \le \varepsilon\). Analogously to the proof of Theorem~\ref{thm:num_evaluations_csmall}, we define \(\Gamma^{t} = \min\{1 + K_1/s + K_2, G_3^t - G_3^{t+1}\}\), let \(T\) be the first time \(t\) when \(Z^t = 0\) or $G_3^t =0$, and observe that this actually implies $Z^T =0$.
    Moreover, we define \(T'\) as the first time when \(Z^t / n > 2\varepsilon \).
    
    By Corollary~\ref{cor:large_s_drift_g} the drift at any time \(t < \min\{T, T'\}\) is at least 
    \[
        \EE\left[ \indicator{t < \min\{T,T'\}} \Gamma^t \mid x, \lambda \right] \ge \delta.
    \] 
    As in the proof of Theorem~\ref{thm:num_evaluations_csmall}, we use Lemma~\ref{lem:improvements_subgaussian} to see that $\delta - \Gamma^t$ is sub-Gaussian since \(h_3\) is decreasing may not increase too much at each step.
 
    In particular, Theorem~\ref{thm:subgaussian_hit_fast} gives that, for a suitable constant $D>0$, the event $ E := \{T > Dn \text{ and } \sum_{\tau=0}^{Dn} \Gamma^\tau <\eps n + K_1\log\lambda_{\max}+K_2\}$ has probability $\Pr[E] = e^{-\Omega(n)}$. If the second event does not happen, $\sum_{\tau=0}^{Dn} \Gamma^\tau \ge \eps n + K_1\log\lambda_{\max}+K_2$, then by the Sandwich Lemma~\ref{lem:sandwich_near_opt} this implies $Z^t \le 0$ for $t = Dn$ and thus $T\le Dn$. Hence, $\Pr[T \le Dn] \ge \Pr[\bar E] = 1- e^{-\Omega(n)}$, and the statement about the number of generations is proven. For the number of function evaluations, in the proof in Theorem~\ref{thm:num_evaluations_csmall} we use the potential function as a black box (except for the Sandwich Lemma), so the proof carries over.
    
\end{proof}



\section{Small Success Rates Yield Exponential Runtimes}
\label{sec:inefficient}

The aim of this section is to show that for large \(s\), that is, for a small enough success rate, the \saolea needs super-polynomial time to find the optimum of any dynamic monotone function. The reason is that the algorithm has negative drift in a region that is still far away from the optimum, in linear distance. In fact, as we have shown in Section~\ref{sec:efficient_near_optimum}, the drift is \emph{positive} close to the optimum. Thus the hardest region for the \saolea is not around the optimum. This surprising phenomenon was discovered for \onemax in~\cite{hevia2021self}. We show that it is not caused by any specific property of \onemax, but that it occurs for \emph{every} dynamic monotone function. Even in the \onemax case, our result is slightly stronger than~\cite{hevia2021arxiv}, since they show their result only for $1 < F < 1.5$, while ours holds for all $F>1$. On the other hand, they give an explicit constant $s_1 =18$ for \onemax.

\begin{theorem}
\label{thm:Inefficient}
Let \(0 < c \le 1 < F \). For every \(\varepsilon >0\), there exists \(s_1 > 0\) such that for all \(s \ge s_1\) the following holds. For every dynamic monotone function and every initial search point \(\xinit\) satisfying \(\ZM(\xinit) \ge \varepsilon n\) the number of generations of the \saolea with success rate $s$, update strength $F$, and mutation probability $c/n$ is \(e^{\Omega( n / \log^2 n )}\) with high probability.
\end{theorem}

\begin{definition}[Potential function for negative result]
    \label{def:h_negative}
    Given $F$, we define 
    \[
        h_4(\lambda) := - K_4 \log_F^2 (\lambda F)  =-K_4 \cdot (\log_F(\lambda) + 1)^2
    \] 
    with $K_4$ a positive constant to be chosen later. 
    As before, we define the potential function to be the sum of \(\ZM(x)\) and \(h_4(\lambda)\): 
    \begin{align*}
        g_4(x, \lambda) = \ZM(x) + h_4(\lambda).
    \end{align*}
\end{definition}
As usual, we set $G_4^t:= g_4(x^t, \lambda^t)$, $H_4^t := h_4(\lambda^t)$ and $Z^t := \ZM(x^t)$. Contrary to the previous sections, we now are now aiming to show that the difference \(G_4^{t+1} - G_4^t\) is positive in expectation. (Note the switched order of~$t+1$ and $t$.) This will require approaches slightly different from the ones we used so far.

The theorem will be proved using the following lemmas. Recall from Section~\ref{sec:basic} the event $B$ that at least one child flips a zero-bit.

\begin{lemma}
    \label{lem:drift_Z_inefficient_1}
    There exists a constant \(\alpha_1 >0\) depending only on \(c\) such that at all times $t$ we have 
    \begin{align*}
        \EE [ Z^{t+1} - Z^{t} \mid x, \lambda ] \ge - \Pr[B] \alpha_1(1 + \log \lambda).
    \end{align*}
\end{lemma}

\begin{lemma}
    \label{lem:drift_Z_inefficient_2}
    There exist constants \(\varepsilon, \alpha_2 > 0\) depending only on \(c, F\) such that if \(Z^t \le \varepsilon n\) and \(\lambda \le F\), then 
    \begin{align*}
        \EE [ Z^{t+1} - Z^{t} \mid x, \lambda ] \ge \alpha_2.
    \end{align*}
\end{lemma}

\begin{lemma}
    \label{lem:drift_H4}
    Assume that \(s \ge 1 \ge c\). 
    At all times \(t\) with $Z^t >0$ we have
    \begin{align*}
        \EE [ H^{t+1} - H^{t} \mid x, \lambda ] \ge \frac{1}{3} \Pr[B] K_4 (1 + \log_F \lambda) \indicator{\lambda \ge F} - \frac{3}{s} K_4 (1 + \log_F \lambda)
    \end{align*}
\end{lemma}

\begin{proof}[Proof of Lemma~\ref{lem:drift_Z_inefficient_1}]
    The event $\bar B$ implies $\supp(x^{t+1}) \subseteq \supp(x^{t})$ and thus $Z^{t+1} -Z^t \ge 0$. Hence, $\EE [ Z^{t+1} - Z^{t} \mid \bar B ] \ge 0$. By the law of total probability, we may thus bound
    \begin{align}\label{eq:lem:drift_Z_inefficient_11}
         \EE [ Z^{t+1} - Z^{t}] & = \Pr[B] \cdot \EE [ Z^{t+1} - Z^{t}\mid B] + \Pr[\bar B] \cdot \EE [ Z^{t+1} - Z^{t}\mid \bar B] \nonumber\\
         & \ge \Pr[B]\cdot \EE[Z^{t+1}-Z^t \mid B].
    \end{align}
    To bound the conditional expectation, let $N^j$ be the number of zero-bits flipped by the $j$-th individual, and let $N := \max_j\{N^j\}$. We have \(Z^{t+1}-Z^t \ge -N\), so we would like to bound $\EE[-N \mid B]$. The events $B^j$ are positively correlated with the event $N\ge z$, for every $z \ge 1$. Therefore,
    \begin{align*}
        \Pr[N \ge z\mid B] & \le \Pr[N \ge z \mid B,B^1,\ldots,B^{\lfloor \lambda \rceil}] = \Pr[N \ge z \mid B^1,\ldots,B^{\lfloor \lambda \rceil}] \\
        & = 1-\prod_{j=1}^{\lfloor \lambda \rceil}(1-\Pr[N^j \ge z \mid B^j]).
    \end{align*}
    As in the proof of Claim~\ref{clm:conditional_drift_2}, we can couple the one-bit flips in $y^j$ given $B^j$ by first sampling the position $l$ of the left-most one-bit flip, and then flipping all bits to the right of $l$ independently with probability $c/n$. Since there are less than $n$ positions to the right of $l$, this shows that $N^j$ is dominated by~$1+N'$,  where $N'$ follows a $\text{Bin}(n,c/n)$ distribution. In particular, by the Chernoff bound, Theorem~\ref{thm:Chernoff}, $\Pr[N^j \ge z \mid B^j] \le \Pr[N' \ge z-1] \le e^{- \alpha_0 (z-1)}$ for a constant $\alpha_0$ that only depends on $c$. Hence,
    \begin{align*}
        \Pr[N \ge z \mid B] 
        & \le 1- \prod_{j=1}^{\lfloor \lambda \rceil}(1-\Pr[N^j \ge z \mid B^j]) \le 1-(1-e^{-\alpha_0(z-1)})^{\lfloor \lambda \rceil}\\
        &\le \min\{1,\lfloor \lambda \rceil e^{-\alpha_0(z-1)}\},
    \end{align*}
    and 
    \begin{align*}
        \EE[N \mid B] 
        & = \sum_{z=1}^{\infty} \Pr[N\ge z \mid B]
        \le \sum_{z=1}^{\infty} \min\{1,\lfloor \lambda \rceil e^{-\alpha_0(z-1)}\} \\
        & \le 1+\log\lfloor\lambda \rceil + \sum_{z=\log\lfloor\lambda \rceil + 1}^{\infty} \lfloor \lambda \rceil e^{-\alpha_0(z-1)} \le \alpha_1(1+\log \lambda)
    \end{align*}
    for a suitable constant $\alpha_1 >0$. Combining this with~\eqref{eq:lem:drift_Z_inefficient_11}, we obtain 
    \begin{align*}
         \EE [ Z^{t+1} - Z^{t}] & \ge \Pr[B]\cdot \EE[Z^{t+1}-Z^t \mid B] \\
         & \ge \Pr[B]\cdot \EE[-N \mid B] \ge - \Pr[B]\alpha_1(1+\log \lambda),
    \end{align*}
    as desired.
\end{proof}

\begin{proof}[Proof of Lemma~\ref{lem:drift_Z_inefficient_2}]
    For all \(j \in [\round{\lambda}]\), let us denote by \(M^j\) the number of one-bits flipped by the \(j\)-th offspring and \(M = \min_j M^j\). We also define \(N^j\) as the number of zero-bits flipped by the \(j\)-th child and let \(N = \max_j N^j\). 
    
    Clearly, \(Z^{t+1} - Z^{t} \ge M - N\), so it suffices to prove that \(\EE \left[ M - N \right] \ge \alpha_2 \) for some constant \(\alpha_2\). 
    
    Observe that \(M\) is the minimum of \(\round{\lambda} \le \round F\) i.i.d.\ random variables following a binomial distribution \(\mathrm{Bin}( n - Z^t, c/n )\). In particular, 
    \begin{align*}
        \Pr[M \ge 1] 
            &= \left(1 - (1 - c/n)^{(n - Z^t)}\right)^{\round{\lambda}} \\
            &\ge \left( 1 - e^{ -(1 - \varepsilon)c } \right)^{\round{F}},
    \end{align*}
    since \(Z^t \le \varepsilon n\). From this we deduce that \(\EE\left[ M \right] \ge ( 1 - e^{ -(1 - \varepsilon)c } )^{\lfloor{F}\rceil} = \Omega(1)\).
    
    Observe now that \(N \le \sum_{j} N^j\). Since each \(N^j\) follows a binomial distribution \(\mathrm{Bin}(Z^t, c/n)\) the expected value of \(N\) is at most 
    \begin{align*}
        \EE \left[ N \right] \le \round{\lambda} Z^t c/n \le \varepsilon c \round{F}.
    \end{align*}
    Choosing \(\varepsilon\) small enough and \(\alpha_2 = \left( 1 - e^{ -(1 - \varepsilon)c } \right)^{\round{F}} - \varepsilon c \round{F}\) gives the result.
\end{proof}

\begin{proof}[Proof of Lemma~\ref{lem:drift_H4}]
    Conditioned on \(f(x^{t+1}) \le f(x^{t})\) we have
    \begin{align*}
        H_4^{t+1} - H_4^{t}
            &= - K_4 \log_F^2 (\lambda F^{1 + 1/s}) + K_4 \log_F^2 (\lambda F)\\
            &= - K_4 \left( \log_F^2 (\lambda F) + 2 \log_F (\lambda F) / s + 1/s^2 - \log_F^2 (\lambda F) \right) \\
            &= - \frac{1}{s} K_4 \left( 2 \log_F \lambda + 2 + 1/s \right) \\
            &\ge - \frac{3}{s} K_4 \left( 1 + \log_F \lambda\right),
    \end{align*}
    since \(s \ge 1\).
    
    If we now condition on \(f(x^{t+1}) > f(x^t)\) and assume \(\lambda \ge F\), we have 
    \begin{align*}
        H_4^{t+1} - H_4^{t}
            &= - K_4 \log_F^2 (\lambda) + K_4 \log_F^2 (\lambda F) \\
            &= K_4 \left( 2\log_F \lambda + 1 \right) \\
            &\ge K_4 ( 1 + \log_F \lambda).
    \end{align*}
    We observe that \(h_4\) is decreasing with \(\lambda\), so when \(\lambda < F\) we may simply lowerbound the drift by \(0\).
    
    The law of total probability then gives 
    \begin{align*}
        \EE \left[ H^{t+1}_4 - H^t_4 \right]
            &\ge \Pr[ f(x^{t+1}) > f(x^t) ] K_4 ( 1 + \log_F \lambda) \indicator{\lambda \ge F} \\
            & \qquad \qquad \qquad - \Pr[f(x^{t+1}) \le f(x^t)] \frac{3}{s} K_4 (1+\log_F \lambda).
    \end{align*}
    Clearly \(\Pr[ f(x^{t+1}) \le f(x^t) ] \le 1\), so to obtain the result is suffices to prove that \(\Pr[ f(x^{t+1}) > f(x^t) ] \ge \Pr[B] / 3 \). 
    
    Recall that \(B\) is the event that some offspring flips a zero-bit of the parent into a one-bit at time \(t\). Assume that \(B\) holds, and let \(j \in [\round{\lambda}]\) be the index of a child which flips a zero-bit of the parent. Clearly, if \(y^j\) flips no one-bit of \(x^t\) into a zero-bit, then \(f(y^j) > f(x)\) and the fitness increases at step \(t\). The event that \(y^j\) flips no 
    one-bit is independent of \(B\) and has probability \((1 - c/n)^{n - Z^t} \ge e^{-c}\). In particular, this implies that \(\Pr[ f(x^{t+1}) > f(x^t) ] \ge e^{-c} \Pr[B] \ge \Pr[B] / 3\) since \(c \le 1\).
\end{proof}

\begin{corollary}\label{cor:non-efficient}
    For all $0 < c\le 1 < F$ and every sufficiently small $\eps >0$ there exists \(s_1 >0\) such that for all \(s \ge s_1\) the following holds. There exists a constant \(\delta > 0\) such that if \(\varepsilon n/2 \le Z^t \le \varepsilon n\) then 
    \begin{align*}
        \EE \left[ G^{t+1}_4 - G^{t}_4 \mid x, \lambda \right] \ge \delta.
    \end{align*}
\end{corollary}

\begin{proof}
    We take $\eps>0$ so small that Lemma~\ref{lem:drift_Z_inefficient_2} is applicable, and let \(\alpha_1, \alpha_2\) be the other constants from Lemmas~\ref{lem:drift_Z_inefficient_1} and \ref{lem:drift_Z_inefficient_2}.
    
    We will show that for a sufficiently large (but constant) \(s\), the drift of \(G^t_4\) is at least a constant \(\delta\) when \(Z^t \in [ \varepsilon n/2, \varepsilon n ]\). We distinguish on whether \(\lambda\) is small or not.
    
    \textit{If \(\lambda \ge F\),} then Lemmas~\ref{lem:drift_Z_inefficient_1} and \ref{lem:drift_H4} combine into 
    \begin{align*}
        \EE [ G^{t+1}_4 - G^t_4 \mid x, \lambda ]
            &\ge \Pr[B] \left( \tfrac{1}{3}K_4(1 + \log_F \lambda) - \alpha_1(1 + \log \lambda) \right) \\
                &\qquad \qquad \qquad- \tfrac{3}{s}K_4(1 + \log_F \lambda).
    \end{align*}
    We choose \(K_4\) large enough so that \(\alpha_1 (1 + \log \lambda) \le K_4(1 + \log_F \lambda) / 12 \) for all \(\lambda\). The drift is then 
    \begin{align}\label{eq:cor:non-efficient}
        \EE [ G^{t+1}_4 - G^t_4 \mid x^t, \lambda^t ]
            &\ge \tfrac{1}{4} \Pr[B] K_4(1 + \log_F \lambda) - \tfrac{3}{s}K_4(1 + \log_F \lambda) \nonumber\\
            &=K_4(1 + \log_F \lambda) \left( \tfrac{1}{4} \Pr[B] - \tfrac{3}{s} \right).
    \end{align}
    Recall that Lemma~\ref{lem:properties_AB} guarantees that \(\Pr[B] \ge 1 - e^{-b_3 Z^t \lambda / n}\) for some positive constant \(b_3\). In particular, since \(Z^t \ge \varepsilon n / 2\), 
    \[
        \Pr[B] \ge 1 - e^{-b_3 \varepsilon \lambda / 2} \ge 1 - e^{-b_3 \varepsilon / 2}
    \]
    is at least a constant. For a choice of \(s\) larger than $24(1-e^{-b_3\eps/2})$, the bound~\eqref{eq:cor:non-efficient} is at least \(\EE [ G^{t+1}_4 - G^t_4 \mid x, \lambda ] \ge \delta\) for some constant \(\delta > 0\).

    \textit{If \(\lambda < F\),} then Lemmas~\ref{lem:drift_Z_inefficient_2} and \ref{lem:drift_H4} guarantee that 
    \begin{align*}
        \EE [ G^{t+1}_4 - G^t_4 \mid x, \lambda ]
            &\ge \alpha_2 - \frac{3}{s} K_4(1 + \log_F \lambda) \\
            &\ge \alpha_2 - \frac{6}{s} K_4.
    \end{align*}
    For a choice of \(s\) large enough this is at least \(\delta\), for some constant \(\delta > 0\). 
    \end{proof}

We are now ready to prove the main theorem of this section. Essentially, it follows from Corollary~\ref{cor:non-efficient} and the Negative Drift Theorem~\ref{thm:NegativeDriftWithScaling}. However, compared to the other sections, there is a slight complication since the difference $|G_4^t-Z^t|= K_4 \log^2_F(\lambda F)$ is not bounded. 
However, we will prove that with overwhelming probability the difference does not grow larger than~$K_4\sqrt{n}$.

\begin{proof}[Proof of Theorem~\ref{thm:Inefficient}]
Let $\Lambda := n^2/F$ and let $T$ be the first point in time when $Z^t \le \eps n/2$. We first show that with overwhelming probability, we have $\lambda^t \le \Lambda$ for all~$1\le t\le \min\{T,e^{n}\}$. Indeed, to obtain some $\lambda > \Lambda$, it would be necessary to have a step with $\lambda > \Lambda F^{-1/s}$ that does not improve the fitness. If this were to happen before time $T$, it must happen in a step with $Z^t \ge \eps n/2$. By Lemma~\ref{lem:properties_AB}, the probability to have a non-improving step is $e^{-\Omega(\lambda)}$. By a union bound, the probability that such a step happens before time $e^n$ is at most $e^{n-\Omega(\lambda)} = o(1)$. Hence, w.h.p.\! $\lambda^t \le \Lambda$ for all $1\le t\le \min\{T,e^{n}\}$. Note that in this case we have $|G_4^t-Z^t| \le 4K_4 \log_F^2n$, so in particular, $G_4^t > 4 K_4 \log_F^2 n$ implies that $Z^t >0$ for $\lambda \le \Lambda$. 

In the following, we will apply the Negative Drift Theorem~\ref{thm:NegativeDriftWithScaling} to $G_4^t$. The drift condition is satisfied by Corollary~\ref{cor:non-efficient} whenever $Z^t\in [\eps n/2, \eps n]$, which is implied whenever $G_4^t \in [\eps n/2 + 4K_4\log_F^2 n, \eps n]$ and $\lambda\le \Lambda$. 

For the step size condition, let $L^j$ denote the total number of bits flipped in $y^j$, and $L := \max\{L^j\}_j$. Since $L^j$ follows a $\text{Bin}(n,c/n)$ distribution, by the Chernoff bound, Theorem~\ref{thm:Chernoff}, there is a constant $\beta >0$ such that $\Pr[L^j \ge z] \le e^{-\beta z}$ for all $z\ge 0$. Let $r:= 4 K_4 \log_F n/\beta$, and note that we can achieve $|H_4^{t+1}-H_4^t| \le r/2$ when $\lambda \le \Lambda$, by making $\beta >0$ smaller if necessary. Then for all $j\ge 1$,
\begin{align*}
    \Pr[|G_4^{t+1} - G_4^t| \ge j r] 
        & \le \Pr[ |Z^{t+1} - Z^t| \ge (j-1/2)r ]
        \le \Pr[L \ge (j-1/2)r] \\
        &= 1-\big(1-\Pr[L^1 \ge (j-1/2)r]\big)^{\lfloor \lambda \rceil} \\
        & \le 1-(1-e^{-\beta (j-1/2)r})^{\lfloor \lambda \rceil} 
        \le 1-(1-n^{-4 (j-1/2)})^{\Lambda+1} \\
        &\le 1- e^{-\tfrac12(\Lambda+1)n^{-4 (j-1/2)}} 
        = 1-e^{-n^{-\Omega(j)}} \\
        &= n^{-\Omega(j)}\le e^{-j},
\end{align*}
where the last inequality holds for $n$ sufficiently large. Thus the step size condition of Theorem~\ref{thm:NegativeDriftWithScaling} is satisfied, and we obtain that w.h.p.\! $G_4^t \geq \eps n/2 + 4K_4\log_F^2 n$ for $e^{\Omega(n/\log^2 n)}$ steps if $\lambda^t \le \Lambda$ during this time. Since the latter also holds w.h.p., this implies $T= e^{\Omega(n/\log^2 n)}$ w.h.p., which concludes the proof.

\end{proof}


\section{Simulations}\label{sec:simulations}
In this section, we provide simulations that complement our theoretical analysis. The functions optimized in our simulations include \onemax, \binary, \hottopic \cite{lengler2019general}, \BinaryValue, and \dynBV \cite{lengler2020large}, where \binary is defined as
$f(x) = \sum_{i=1}^{\lfloor n/2 \rfloor} x_i n + \sum_{i=\lfloor n/2 \rfloor + 1}^n x_i$, and $\BinaryValue$ is defined as $f(x) = \sum_{i=1}^{n} 2^{i-1} x_i $. The definition of \hottopic can be found in~\cite{lengler2019general}, and we set the parameters to $L=100$, $\alpha=0.25$, $\beta=0.05$, and~$\eps=0.05$. \dynBV is the dynamic environment which applies the \BinaryValue function to a random permutation of the $n$ bit positions, see~\cite{lengler2020large} for its formal definition. In all experiments, we start the \saoclea with a randomly sampled search point and an initial offspring size of $\lambda^{\text{init}} = 1$. The algorithm terminates when the optimum is found or after $500 n$ generations. The code for the simulations can be found at \url{https://github.com/zuxu/OneLambdaEA}.

\subsection{Threshold of $s$}

In Figure~\ref{fig:threshold}, we follow the same setup as in \cite{hevia2021arxiv}, but for a larger set of functions. We observe exactly the same threshold $s=3.4$ for \onemax. For the other monotone functions of our choice, the threshold effect happens before $s=3.4$, which suggests that some hard monotone functions might have a lower allowance for the value of $s$ than \onemax, other than conjec\-tured by Hevia Fajardo and Sudholt in~\cite{hevia2021arxiv}.

\begin{figure}[ht]
    \centering
    \includegraphics[width=0.8\textwidth]{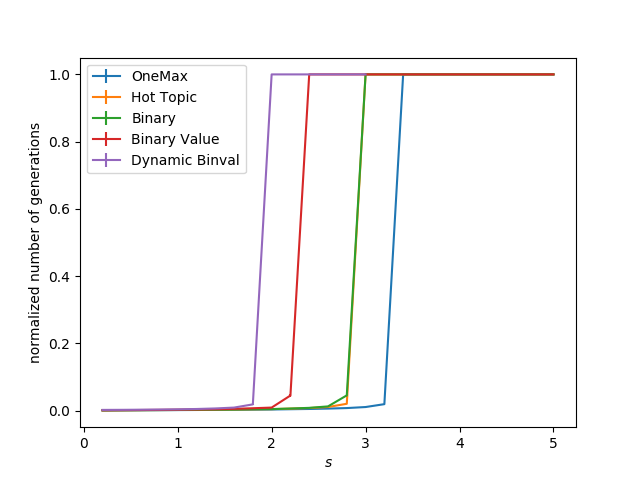}
    \caption{Average number of generations of the self-adjusting \oclea with $F = 1.5$ and $c=1$ in 10 runs when optimizing monotone functions with $n=10000$, normalised and capped at $500 n$ generations. \hottopic and \binary behave similarly, so the corresponding curves have a large overlap. The evaluated values of $s$ range from 0.2 to 5 with a step size of 0.2 for all functions except that \dynBV was not evaluated for $3.2 \le s\le 5$ due to performance issues.}
    \label{fig:threshold}
\end{figure}

\subsection{Effect of $F$}
We have shown that the \saoclea with $c< 1$ optimizes every dyna\-mic monotone function efficiently when $s$ is sufficiently small and is inefficient when $s$ is too large. Both results hold for arbitrary $F$. It is natural to assume that there is a threshold $s_0$ between the efficient and inefficient regime. However, Figure~\ref{fig:fDynamic} below shows that the situation might be more complicated. For this plot, we have first empirically determined an efficiency threshold for $s$ on \dynBV (see Figure~\ref{fig:threshold}), then fixed $s$ slightly below this threshold and systematically varied the value of $F$. For this intermediate value of $s$, we see that there is a phase transition in terms of~$F$. 

Hence, we conjecture that there is no threshold $s_0$ such that the \saoclea is efficient for all $s < s_0$ and all $F>1$, and inefficient for all $s > s_0$ and all $F>1$. Rather, we conjecture that there is `middle range' of values of $s$ for which it depends on the value of $F$ whether the \saolea is efficient. Note that we know from this paper that this phenomenon can \emph{only} occur for a `middle range': both for sufficiently small $s$ (Theorems~\ref{thm:generations}, \ref{thm:num_evaluations_csmall}), and for sufficiently large~$s$ (Theorem~\ref{thm:Inefficient}), the value of $F$ does not play a role.

In general, smaller values of $F$ seem to be beneficial. However, the correlation is not perfect, see for example the dip for $c = 0.98$ and $F=5.5$ in the left subplot of Figure~\ref{fig:fDynamic}. These dips also happen for some other combinations of $s, F$ and $c$ (not shown), and they seem to be consistent, i.e., they do not disappear with a larger number of runs or larger values of $n$ up to $n=5000$. To test whether this is due to the rounding scheme, we checked whether the effect disappears if we round $\lambda$ in each generation stochastically to the next integer; e.g., $\lambda^t = 2.6$ means that in generation~$t$ we create two offspring with probability $40\%$ and three offspring with probability $60\%$. The effect remains, and the runtime still seems to depend on $F$ in a non-monotone fashion, see the right subplot of Figure~\ref{fig:fDynamic}.

The impact of $F$ is visible for all ranges $c < 1$, $c=1$ and $c>1$. For $c=1$ we have only proven efficiency for sufficiently small $F$. However, we conjecture that there is no real phase transition at $c=1$, and the `only' difference is that our proof methods break down at this point. 
For the fixed~$s$, with increasing $c$ the range of $F$ becomes narrower and restricts to smaller values while larger values of $c$ admit a larger range of values for $F$.

\begin{figure}[ht]
    \centering
    \includegraphics[width=\textwidth]{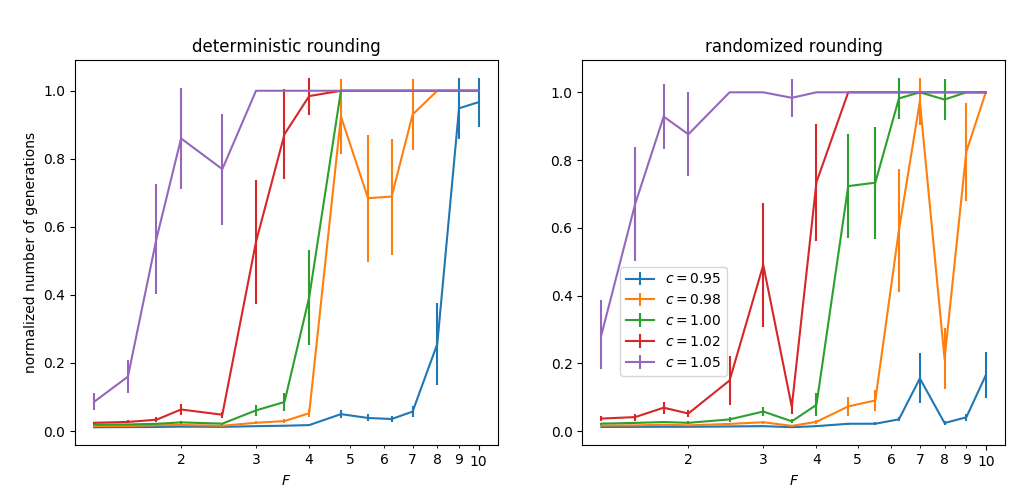}
    \caption{Average number of generations of the self-adjusting \oclea with $s=1.8$ and in 50 runs when optimizing \dynBV with $n=1000$, normalised and capped at $500 n$ generations. The left and right subplots correspond to the deterministic and randomized rounding schemes respectively. The length of each vertical bar is one standard deviation at its corresponding data point.}
    \label{fig:fDynamic}
\end{figure}

\section{Conclusion}

In this paper, we have studied the \saolea on dynamic monotone functions. Hevia Fajardo and Sudholt had shown an extremely strong dependency of the performance on the success rate $s$ for the \onemax benchmark. We have shown that there is nothing specific to \onemax about the situation. The same effect happens for any (static or dynamic) monotone fitness function: for small  values of $s$, the \saolea is efficient on all dynamic monotone functions, while for large values of $s$, the \saolea is inefficient on every dynamic monotone function. In the latter case, the bottleneck is not around the optimum, but rather in some area of linear distance from the optimum. Thus the \saolea is one of the surprising examples showing that some algorithms may fail in easy fitness landscapes, but succeed in hard fitness landscapes.

Hevia Fajardo and Sudholt have conjectured that the problem becomes worse the easier the fitness landscape is. Concretely, they conjectured that any parameter choice that works for \onemax should also give good result for any other landscape~\cite{hevia2021self}. In a companion paper~\cite{kaufmann2022onemax}, we disprove this conjecture, but for an unexpected reason: there are different ways to measure `easiness' of a fitness landscape. While it is theoretically proven that \onemax is the easiest fitness function with respect to decreasing the distance from the optimum~\cite{doerr2012multiplicative}, this is not the aspect that matters for the \saolea. Here, the important aspect is how easy it is to find a fitness improvement, since this may induce too small target population sizes in the \saolea. For finding fitness improvements, there are easier functions than \onemax, for example the dynamic BinVal function~\cite{lengler2020large} or HotTopic functions~\cite{lengler2019general}, see~\cite{kaufmann2022onemax} for details. It remains open to determine the easiest dynamic monotone function $f_{\text{easiest}}$ with respect to fitness improvements. A candidate for $f_{\text{easiest}}$ might be the `adversarial' \textsc{Dynamic BinVal}, which we define as \textsc{Dynamic BinVal} (see Section~\ref{sec:simulations}) with the exception that the permutation is not random but chosen so that any \(0\)-bit is heavier than any~\(1\)-bit. With this fitness function, any \(0\)-bit flip gives a fitter child, regardless of the number of \(1\)-bit flips, so it is intuitively convincing that it should be the easiest function with respect to fitness improvement.

Moreover, the conjecture of Hevia Fajardo and Sudholt might still hold if we replace \onemax by $f_{\text{easiest}}$. I.e., is it true that any parameter choice that works for $f_{\text{easiest}}$ also works for any other dynamic monotone function, and perhaps even in yet more general settings?

Apart from that, the most puzzling part of the picture is the experimental finding that in a `middle regime' of success rates, the update strength $F$ seems to play a role in a non-monotone way (for fixed success rate $s$). It is open to prove theoretically that there is indeed such a `middle regime' where $F$ plays a role at all. For why this effect is non-monotone in $F$, we do not even have a good hypothesis. As outlined in Section~\ref{sec:simulations}, it does not seem to be a rounding effect. This shows that we are still missing important parts of the overall picture.

\subsubsection*{Acknowledgements.} We thank Dirk Sudholt for helpful discussions during the Dagstuhl semi\-nar 22081 ``Theory of Randomized Optimization Heuristics'', and we thank Mario Hevia Fajardo for sharing his code for comparison.

We also thank the two anonymous reviewers for their valuable input and suggestions. Their feedback contributed to improving the quality of this paper, and we are grateful for their time and effort.

 \bibliographystyle{elsarticle-num} 
 \bibliography{cas-refs}

\begin{thebibliography}{10}
\expandafter\ifx\csname url\endcsname\relax
  \def\url#1{\texttt{#1}}\fi
\expandafter\ifx\csname urlprefix\endcsname\relax\def\urlprefix{URL }\fi
\expandafter\ifx\csname href\endcsname\relax
  \def\href#1#2{#2} \def\path#1{#1}\fi

\bibitem{kaufmann2022self}
M.~Kaufmann, M.~Larcher, J.~Lengler, X.~Zou, Self-adjusting population sizes for the {$(1, \lambda)$-{EA}} on monotone functions, in: Parallel Problem Solving from Nature (PPSN), Springer, 2022.

\bibitem{Eiben99parameter}
A.~E. Eiben, R.~Hinterding, Z.~Michalewicz, Parameter control in evolutionary algorithms, IEEE Transactions on Evolutionary Computation 3 (1999) 124--141.

\bibitem{jagerskupper2007plus}
J.~J{\"a}gersk{\"u}pper, T.~Storch, When the plus strategy outperforms the comma strategy and when not, in: Foundations of Computational Intelligence (FOCI), IEEE, 2007, pp. 25--32.

\bibitem{rowe2014choice}
J.~E. Rowe, D.~Sudholt, The choice of the offspring population size in the (1, $\lambda$) evolutionary algorithm, Theoretical Computer Science 545 (2014) 20--38.

\bibitem{antipov2019efficiency}
D.~Antipov, B.~Doerr, Q.~Yang, The efficiency threshold for the offspring population size of the $(\mu, \lambda)$ {EA}, in: Genetic and Evolutionary Computation Conference (GECCO), 2019, pp. 1461--1469.

\bibitem{doerr2010optimizing}
B.~Doerr, T.~Jansen, D.~Sudholt, C.~Winzen, C.~Zarges, Optimizing monotone functions can be difficult, in: International Conference on Parallel Problem Solving from Nature, Springer, 2010, pp. 42--51.

\bibitem{doerr2013mutation}
B.~Doerr, T.~Jansen, D.~Sudholt, C.~Winzen, C.~Zarges, Mutation rate matters even when optimizing monotonic functions, Evolutionary Computation 21~(1) (2013) 1--27.

\bibitem{lengler2018drift}
J.~Lengler, A.~Steger, Drift analysis and evolutionary algorithms revisited, Combinatorics, Probability and Computing 27~(4) (2018) 643--666.

\bibitem{badkobeh2014unbiased}
G.~Badkobeh, P.~K. Lehre, D.~Sudholt, Unbiased black-box complexity of parallel search, in: Parallel Problem Solving from Nature (PPSN), Springer, 2014, pp. 892--901.

\bibitem{bottcher2010optimal}
S.~B{\"o}ttcher, B.~Doerr, F.~Neumann, Optimal fixed and adaptive mutation rates for the {LeadingOnes} problem, in: Parallel Problem Solving from Nature (PPSN), Springer, 2010, pp. 1--10.

\bibitem{doerr2015black}
B.~Doerr, C.~Doerr, F.~Ebel, From black-box complexity to designing new genetic algorithms, Theoretical Computer Science 567 (2015) 87--104.

\bibitem{doerr2020optimal}
B.~Doerr, C.~Doerr, J.~Yang, Optimal parameter choices via precise black-box analysis, Theoretical Computer Science 801 (2020) 1--34.

\bibitem{doerr2021runtime}
B.~Doerr, C.~Witt, J.~Yang, Runtime analysis for self-adaptive mutation rates, Algorithmica 83~(4) (2021) 1012--1053.

\bibitem{karafotias2014parameter}
G.~Karafotias, M.~Hoogendoorn, {\'A}.~E. Eiben, Parameter control in evolutionary algorithms: Trends and challenges, IEEE Transactions on Evolutionary Computation 19~(2) (2014) 167--187.

\bibitem{aleti2016systematic}
A.~Aleti, I.~Moser, A systematic literature review of adaptive parameter control methods for evolutionary algorithms, ACM Computing Surveys (CSUR) 49~(3) (2016) 1--35.

\bibitem{doerr2015optimal}
B.~Doerr, C.~Doerr, Optimal parameter choices through self-adjustment: Applying the 1/5-th rule in discrete settings, in: Proceedings of the 2015 Annual Conference on Genetic and Evolutionary Computation, 2015, pp. 1335--1342.

\bibitem{lassig2011adaptive}
J.~L{\"a}ssig, D.~Sudholt, Adaptive population models for offspring populations and parallel evolutionary algorithms, in: Foundations of Genetic Algorithms (FOGA), 2011, pp. 181--192.

\bibitem{doerr2019selfadjusting}
B.~Doerr, C.~Gie{\ss}en, C.~Witt, J.~Yang, The $(1+\lambda)$ evolutionary algorithm with self-adjusting mutation rate, Algorithmica 81~(2) (2019) 593--631.

\bibitem{rajabi2020self}
A.~Rajabi, C.~Witt, Self-adjusting evolutionary algorithms for multimodal optimization, in: Genetic and Evolutionary Computation Conference (GECCO), 2020, pp. 1314--1322.

\bibitem{doerr2021self}
B.~Doerr, C.~Doerr, J.~Lengler, Self-adjusting mutation rates with provably optimal success rules, Algorithmica 83~(10) (2021) 3108--3147.

\bibitem{rodionova2019offspring}
A.~Rodionova, K.~Antonov, A.~Buzdalova, C.~Doerr, Offspring population size matters when comparing evolutionary algorithms with self-adjusting mutation rates, in: Genetic and Evolutionary Computation Conference (GECCO), 2019, pp. 855--863.

\bibitem{lissovoi2019time}
A.~Lissovoi, P.~S. Oliveto, J.~A. Warwicker, On the time complexity of algorithm selection hyper-heuristics for multimodal optimisation, in: AAAI Conference on Artificial Intelligence (AAAI), Vol. 33(1), 2019, pp. 2322--2329.

\bibitem{lissovoi2020duration}
A.~Lissovoi, P.~Oliveto, J.~A. Warwicker, How the duration of the learning period affects the performance of random gradient selection hyper-heuristics, in: AAAI Conference on Artificial Intelligence (AAAI), Vol. 34(3), 2020, pp. 2376--2383.

\bibitem{lissovoi2020simple}
A.~Lissovoi, P.~S. Oliveto, J.~A. Warwicker, Simple hyper-heuristics control the neighbourhood size of randomised local search optimally for {LeadingOnes}, Evolutionary Computation 28~(3) (2020) 437--461.

\bibitem{doerr2018runtime}
B.~Doerr, A.~Lissovoi, P.~S. Oliveto, J.~A. Warwicker, On the runtime analysis of selection hyper-heuristics with adaptive learning periods, in: Genetic and Evolutionary Computation Conference (GECCO), 2018, pp. 1015--1022.

\bibitem{doerr2018static}
B.~Doerr, C.~Doerr, T.~K{\"o}tzing, Static and self-adjusting mutation strengths for multi-valued decision variables, Algorithmica 80~(5) (2018) 1732--1768.

\bibitem{hevia2020choice}
M.~A. Hevia~Fajardo, D.~Sudholt, On the choice of the parameter control mechanism in the (1+($\lambda$, $\lambda$)) genetic algorithm, in: Genetic and Evolutionary Computation Conference (GECCO), 2020, pp. 832--840.

\bibitem{mambrini2015design}
A.~Mambrini, D.~Sudholt, Design and analysis of schemes for adapting migration intervals in parallel evolutionary algorithms, Evolutionary Computation 23~(4) (2015) 559--582.

\bibitem{case2020self}
B.~Case, P.~K. Lehre, Self-adaptation in nonelitist evolutionary algorithms on discrete problems with unknown structure, IEEE Transactions on Evolutionary Computation 24~(4) (2020) 650--663.

\bibitem{rajabi2020evolutionary}
A.~Rajabi, C.~Witt, Evolutionary algorithms with self-adjusting asymmetric mutation, in: Parallel Problem Solving from Nature (PPSN), Springer, 2020, pp. 664--677.

\bibitem{doerr2020theory}
B.~Doerr, C.~Doerr, Theory of parameter control for discrete black-box optimization: Provable performance gains through dynamic parameter choices, Theory of evolutionary computation (2020) 271--321.

\bibitem{hevia2021arxiv}
M.~A. Hevia~Fajardo, D.~Sudholt, Self-adjusting population sizes for non-elitist evolutionary algorithms: Why success rates matter, arXiv preprint arXiv:2104.05624 (2021).

\bibitem{kern2004learning}
S.~Kern, S.~D. M{\"u}ller, N.~Hansen, D.~B{\"u}che, J.~Ocenasek, P.~Koumoutsakos, Learning probability distributions in continuous evolutionary algorithms--a comparative review, Natural Computing 3~(1) (2004) 77--112.

\bibitem{rechenberg1978evolutionsstrategien}
I.~Rechenberg, Evolutionsstrategien, in: Simulationsmethoden in der Medizin und Biologie, Springer, 1978, pp. 83--114.

\bibitem{devroye72compound}
L.~Devroye, The compound random search, Ph.D. dissertation, Purdue Univ., West Lafayette, {IN}, 1972.

\bibitem{schumer1968adaptive}
M.~Schumer, K.~Steiglitz, Adaptive step size random search, IEEE Transactions on Automatic Control 13~(3) (1968) 270--276.

\bibitem{auger2009benchmarking}
A.~Auger, Benchmarking the (1+ 1) evolution strategy with one-fifth success rule on the {BBOB}-2009 function testbed, in: Genetic and Evolutionary Computation Conference (GECCO), 2009, pp. 2447--2452.

\bibitem{hevia2021self}
M.~A. Hevia~Fajardo, D.~Sudholt, Self-adjusting population sizes for non-elitist evolutionary algorithms: why success rates matter, in: Genetic and Evolutionary Computation Conference (GECCO), 2021, pp. 1151--1159.

\bibitem{lehre2012black}
P.~K. Lehre, C.~Witt, Black-box search by unbiased variation, Algorithmica 64 (2012) 623--642.

\bibitem{lengler2021exponential}
J.~Lengler, X.~Zou, Exponential slowdown for larger populations: The $(\mu+ 1)$-{EA} on monotone functions, Theoretical Computer Science 875 (2021) 28--51.

\bibitem{lengler2020large}
J.~Lengler, J.~Meier, Large population sizes and crossover help in dynamic environments, in: Parallel Problem Solving from Nature (PPSN), Springer, 2020, pp. 610--622.

\bibitem{lengler2021runtime}
J.~Lengler, S.~Riedi, {Runtime Analysis of the $(\mu+ 1)$-EA on the Dynamic BinVal Function}, in: Evolutionary Computation in Combinatorial Optimization (EvoCom), Springer, 2021, pp. 84--99.

\bibitem{hevia2022hard}
M.~A. Hevia~Fajardo, D.~Sudholt, Hard problems are easier for success-based parameter control, in: Genetic and Evolutionary Computation Conference (GECCO), 2022, pp. 796--804.

\bibitem{jorritsma2023comma}
J.~Jorritsma, J.~Lengler, D.~Sudholt, Comma selection outperforms plus selection on onemax with randomly planted optima, in: Genetic and Evolutionary Computation Conference (GECCO), 2023.

\bibitem{lengler2019does}
J.~Lengler, A.~Martinsson, A.~Steger, When does hillclimbing fail on monotone functions: An entropy compression argument, in: Analytic Algorithmics and Combinatorics (ANALCO), SIAM, 2019, pp. 94--102.

\bibitem{lengler2019general}
J.~Lengler, A general dichotomy of evolutionary algorithms on monotone functions, IEEE Transactions on Evolutionary Computation 24~(6) (2019) 995--1009.

\bibitem{jansen2007brittleness}
T.~Jansen, On the brittleness of evolutionary algorithms, in: Foundations of Genetic Algorithms (FOGA), Springer, 2007, pp. 54--69.

\bibitem{colin2014monotonic}
S.~Colin, B.~Doerr, G.~F{\'e}rey, Monotonic functions in {EC}: anything but monotone!, in: Genetic and Evolutionary Computation Conference (GECCO), 2014, pp. 753--760.

\bibitem{doerr2018optimal}
B.~Doerr, C.~Doerr, Optimal static and self-adjusting parameter choices for the {$(1+(\lambda,\lambda))$ Genetic Algorithm}, Algorithmica 80~(5) (2018) 1658--1709.

\bibitem{doerr2016provably}
B.~Doerr, C.~Doerr, T.~K{\"o}tzing, Provably optimal self-adjusting step sizes for multi-valued decision variables, in: International Conference on Parallel Problem Solving from Nature, Springer, 2016, pp. 782--791.

\bibitem{kotzing2016concentration}
T.~K{\"o}tzing, Concentration of first hitting times under additive drift, Algorithmica 75~(3) (2016) 490--506.

\bibitem{lengler2020drift}
J.~Lengler, Drift analysis, in: Theory of Evolutionary Computation, Springer, 2020, pp. 89--131.

\bibitem{lengler2018noisy}
J.~Lengler, U.~Schaller, The $(1+1)$-{EA} on noisy linear functions with random positive weights, in: Symposium Series on Computational Intelligence (SSCI), IEEE, 2018, pp. 712--719.

\bibitem{doerr2012multiplicative}
B.~Doerr, D.~Johannsen, C.~Winzen, Multiplicative drift analysis, Algorithmica 64 (2012) 673--697.

\bibitem{oliveto2015improved}
P.~S. Oliveto, C.~Witt, Improved time complexity analysis of the simple genetic algorithm, Theoretical Computer Science 605 (2015) 21--41.

\bibitem{doerr2010drift}
B.~Doerr, L.~A. Goldberg, Drift analysis with tail bounds, in: Parallel Problem Solving from Nature (PPSN), Springer, 2010, pp. 174--183.

\bibitem{grimmett1999percolation}
G.~R. Grimmett, et~al., Percolation, Vol. 321, Springer Science \& Business Media, 1999.

\bibitem{doerr2020probabilistic}
B.~Doerr, Probabilistic tools for the analysis of randomized optimization heuristics, in: Theory of evolutionary computation, Springer, 2020, pp. 1--87.

\bibitem{doerr2015optimizing}
B.~Doerr, M.~K{\"u}nnemann, Optimizing linear functions with the (1+ $\lambda$) evolutionary algorithm—different asymptotic runtimes for different instances, Theoretical Computer Science 561 (2015) 3--23.

\bibitem{kaufmann2022onemax}
M.~Kaufmann, M.~Larcher, J.~Lengler, X.~Zou, \href{https://arxiv.org/abs/2204.07017}{{OneMax} is not the easiest function for fitness improvements} (2022).
\newline\urlprefix\url{https://arxiv.org/abs/2204.07017}

\end{thebibliography}





\end{document}